\documentclass{article} 
\usepackage{iclr2027_conference,times}
\usepackage[utf8]{inputenc}
\usepackage{amsmath,amssymb,amsfonts,amsthm,mathtools}
\usepackage{graphicx}
\usepackage{textcomp}
\usepackage{xcolor}
\usepackage{booktabs}
\usepackage{multirow}
\usepackage{url} 
\graphicspath{{figures/}}

\newcommand{\fig}[4][tbp]{
  \begin{figure}[#1]
    {\centering{\includegraphics[#2]{#3}}\par}
    \caption{#4}\label{fig:#3}
  \end{figure}
}

\theoremstyle{plain}
\newtheorem{proposition}{Proposition}

\theoremstyle{definition}
\newtheorem{assumption}{Assumption}
\theoremstyle{remark}

\newcommand{\enc}{e}
\newcommand{\tenc}{\bar{e}}
\newcommand{\pred}{p}
\newcommand{\Uset}{\mathcal{U}_h}
\newcommand{\R}{\mathbb{R}}
\newcommand{\norm}[1]{\left\lVert#1\right\rVert}

\title{What Do Latent Predictive Vehicle Representations Retain? Measuring State, Geometry, and Local Response}

\author{Enzo Nicol\'as Spotorno \\
Department of Informatics and Statistics \\
Federal University of Santa Catarina \\
\texttt{enzoniko@lisha.ufsc.br}
\AND
Josafat Leal Filho \\
Department of Informatics and Statistics \\
Federal University of Santa Catarina \\
\AND
Ant\^onio Augusto Fr\"ohlich \\
Department of Informatics and Statistics \\
Federal University of Santa Catarina
}

\iclrfinalcopy 
\begin{document}

\maketitle

\begin{abstract} 
Models of vehicle dynamics learned from logged states and commands complement physics-based models, and latent world models, which predict in a learned representation, are used to plan and train controllers in other domains. Vehicle controllers are usually specified in physical terms: costs, limits, and references depend on position, yaw angle, speed, and yaw rate, and the optimizer compares or differentiates predicted outcomes across nearby commands. A latent model placed in such a controller must therefore let these quantities be recovered and must change its predictions with commands as the vehicle does, and prediction error on its own latent targets measures neither. We contribute a measurement protocol for action-conditioned latent predictors with a physical readout that separately tests retention, physical-neighborhood organization, forecasting, and local response to command perturbations, using an untrained-encoder reference and three matched response paths that locate errors in the representation or the predictor. In a case study of a temporal joint-embedding predictive model trained on signals logged in IPG CarMaker, the representations retain the measured planar outputs, though an untrained encoder of the same architecture retains them slightly better; future-command input improves one-second forecasts with retention nearly unchanged; and responses to small command pulses diverge from the simulator already in latent coordinates, raising regret when choosing among nearby commands in all comparisons. Updating the predictor on responses corrects them locally at a cost in forecast accuracy. Measuring retention, forecasting, and local response separately is thus what qualifies a predictive latent as a candidate model for control, and the protocol provides the basis for its closed-loop evaluation.
\end{abstract}

\section{Introduction}\label{sec:intro} 

Vehicle motion control relies on models that predict how position, yaw angle, speed, and yaw rate respond to steering, throttle, and brake commands. Single-track and tire models provide these predictions with interpretable parameters, and their accuracy depends on how well those parameters match the vehicle, its tires, and the road \citep{rajamani2012vehicle,pacejka2012tyre}. Learned and physics-constrained models trained on histories of measured states and commands can capture effects that simplified models omit, including in high-performance and racing conditions \citep{spielberg2019neural,chrosniak2024deepdynamics,rhode2024vehicle}. Latent world models carry learned prediction into a learned representation: they encode observations, predict future embeddings under candidate actions, and have been used to plan and train policies in simulated control tasks \citep{hafner2019planet,hafner2020dreamer,hansen2024tdmpc2}, and driving systems have begun to use latent world-model objectives to improve end-to-end planners \citep{li2025law}.
 
Joint-embedding predictive architectures learn such representations without reconstructing observations: an encoder is trained so that a predictor can match the embedding of a future view \citep{assran2023ijepa,balestriero2025lejepa}, and conditioning the predictor on future controls makes it a controlled transition model \citep{schwarzer2021spr,bagatella2025tdjepa}. Placing such a model in a vehicle controller imposes requirements that this training signal does not address. Vehicle controllers are usually specified in physical quantities: tracking costs, stability and actuation limits, and references are written in position, yaw angle, speed, and yaw rate \citep{rajamani2012vehicle}, so a latent model in such a controller must let these quantities be recovered from its representation. Cost estimates, local approximations, and interpolation between visited conditions are more reliable when physically similar conditions remain close after encoding. The optimizer also depends on how predictions change with the command: sampling-based planners rank nearby candidate commands by predicted cost, and gradient- or linearization-based methods use the derivative of the prediction with respect to the command. An error in that response can change the chosen command even when average forecast error is small. Agreement with the model's own latent targets measures none of these three properties: retention of physical quantities, their organization, and local command response.
 
We contribute a measurement protocol that tests these properties separately for any action-conditioned latent predictor with a readout to physical quantities (Figure~\ref{fig:overview}b). Frozen readouts, compared with an untrained encoder of the same architecture, measure retention, physical-neighbor overlap and a collision check measure organization, and forecasts are compared with persistence and with a supervised forecaster that uses the same inputs. Local response to small command perturbations is compared along three matched paths, the plant's recorded outputs, the decoded embedding of the realized future window, and the decoded prediction, which separate error present in the target representation from error added by prediction, and an offline ranking of candidate commands shows whether response errors change a decision. A formal analysis ties each measurement to the claim it supports (Section~\ref{sec:theory}). Retention, organization, and forecasting need only logged data, and the response lens needs matched command perturbations from a common initial condition, which a simulator provides exactly and repeated test-track runs could approximate.
 
We apply the protocol in a case study chosen so that every lens can be measured: per-vehicle models trained on ego-vehicle signals from continuous IPG CarMaker episodes on one route with one automated driver \citep{ipg2024carmaker}. The model encodes a $0.5$-second history of $22$ recorded channels, including six planar reference outputs, and predicts future embeddings from future steering, throttle, and brake commands. Because these outputs are observed, a readout measures what survives predictive compression rather than estimating a hidden state, and the history and auxiliary channels may carry further information that the objective is free to keep. Five training conditions vary data collection, command availability, and command content in the prediction targets, across six vehicle configurations and three seeds ($90$ fits), so that a difference in retention, organization, or response can be traced to one training change. The model, a temporal convolutional encoder with a controlled neural ODE predictor integrated by RK4 \citep{ulmen2025state}, is an experimental instrument rather than a proposed architecture.
 
We hypothesize that these properties are separable in predictive latents trained on logged driving: training without reconstruction retains the observed outputs, while local command response is not ensured by forecast accuracy and can remain misaligned with the vehicle where forecasts improve on persistence. Three observations motivate this hypothesis. The prediction targets are future windows that contain the future outputs, and a sufficiently accurate predictor cannot merge conditions whose supported futures differ (Section~\ref{sec:theory}); since the outputs evolve continuously, discarding them would raise the prediction loss. Bounded prediction error places no bound on the derivative of the prediction with respect to the command (Section~\ref{sec:theory}). Logged commands come from a driver acting on the vehicle's state, so command and state vary together in the training data; a predictor can then explain futures through the state and learn little about the separate effect of a command, a limitation familiar from closed-loop system identification \citep{forssell1999closed}.

\section{Related Work}\label{sec:related} 

Joint-embedding predictive architectures train an encoder by predicting the embedding of one view from that of another instead of reconstructing observations. LeJEPA defines the family by predictive agreement between views plus non-degenerate embeddings, allows the views to be crops, modalities, or temporal samples, and relates the embedding distribution to the risk of downstream probes \citep{balestriero2025lejepa}. This framing motivates our question: when the views are a vehicle's recent signals and its future under given commands, which physical quantities does the learned embedding keep? Stable training relies on anti-collapse mechanisms, either a moving-average target with an asymmetric predictor \citep{grill2020byol,assran2023ijepa} or variance and covariance penalties \citep{bardes2022vicreg}; \citet{tang2023understanding} analyze how self-predictive learning avoids collapse under idealized linear conditions. Our model combines a moving-average target with VICReg-style penalties, rather than LeJEPA's SIGReg regularizer.
 
Conditioning the predictor on actions turns latent prediction into a controlled transition model. SPR predicts momentum-encoder embeddings through an action-conditioned transition \citep{schwarzer2021spr}, TD-JEPA learns policy-conditioned latent dynamics from reward-free data collected under several policies \citep{bagatella2025tdjepa}, and TS-JEPA applies joint-embedding prediction to time series \citep{ennadir2025tsjepa}; our temporal predictor follows this line, conditioned on logged future driver commands. Such models support decisions in two ways. PlaNet, Dreamer, and TD-MPC2 learn latent dynamics together with reward or value targets and plan or train policies in them \citep{hafner2019planet,hafner2020dreamer,hansen2024tdmpc2}; E2C constrains latent dynamics to be locally linear so that an optimal-control routine can use them \citep{watter2015e2c}. In both cases the controller depends on how predictions change when actions change, which for a model trained only to predict embeddings is an empirical property that we measure directly. Formal accounts specify when prediction yields a state, through action-conditional predictions sufficient for all future tests \citep{littman2001psr}, reward and transition structure that bounds value differences \citep{gelada2019deepmdp,ferns2011bisimulation}, and output histories that determine the state \citep{hermann1977nonlinear,takens1981detecting}, and Section~\ref{sec:theory} adapts the behavioral-distance idea to a reward-free, finite-horizon, support-restricted setting.
 
Recent studies measure physical content in predictive latents directly. \citet{tan2026know} ask which hidden parameters (mass, drag, contact stiffness) a latent world model acquires, and they vary inputs, targets, horizons, and regularization, repeat the analysis on real-robot recordings, and test latent model-predictive control. Retention follows what the prediction target requires, so a modality supplied as input is not retained unless the target requires it, and a parameter recoverable from observations can remain nearly absent from the latent. LeWorldModel trains an action-conditioned joint-embedding world model from pixels with SIGReg, probes it for physical quantities, and plans with it \citep{maes2026lewm}. \citet{ulmen2025state} learn state-space models with a sequence encoder, a controlled neural ODE integrated with RK4, joint-embedding training, and a decoder fitted afterwards on true latent states, evaluated on decoded pendulum images; we build on their controlled-ODE predictor and post-hoc readout. We extend these measurements from content to organization and response. Because our reference outputs are observed in the recorded signals, readouts quantify what survives predictive compression, a question for which an input-access reference plays the role that recoverability certificates play for hidden parameters. Probes show that a quantity is accessible to a readout \citep{anand2019atari,hewitt2019control}, and the three response paths show how the predictive path carries that information when commands change.
 
Vehicle dynamics is a setting where learned models address a known gap. Single-track and tire models describe sideslip, yaw, and tire forces explicitly and remain strong baselines, but their accuracy depends on how well parameters match the vehicle and road \citep{rajamani2012vehicle,pacejka2012tyre}. \citet{spielberg2019neural} trained a neural network on histories of measured states and commands that tracked paths better than a tuned physics model on a real vehicle, across dry and low-friction surfaces without an explicit friction estimate, while the physics model performed better in simulation, where it matched the plant exactly. \citet{rhode2024vehicle} keep known single-track kinematics in a neural differential equation and learn the remaining terms, improving prediction over the simplified physics model with less training data than an unconstrained neural ODE. Our predictor uses the continuous-time controlled form without vehicle equations, so everything it knows about the vehicle is learned from data, which is what our measurements examine. In driving systems, LAW adds a latent world-model loss, conditioned on the planner's own waypoints, to a camera-based end-to-end planner and reports improved planning \citep{li2025law}.
 
Together, these works show that predictive latents can be trained without reconstruction, conditioned on actions, used for planning, and adopted in driving, and that what they retain depends on their inputs and targets. Recovery of physical quantities, their organization, and command response have been examined for hidden physical parameters, decoded images, and task performance, but, to our knowledge, not jointly for an ego-vehicle predictive representation. Ego-vehicle signals make that joint examination tractable, because the reference outputs are measured, commands are explicit model inputs, and matched command perturbations can be replayed in a simulator and compared in physical units.

\section{System Model}\label{sec:model}

\begin{figure}[t]
\centering
\IfFileExists{figures/overview.png}{\includegraphics[width=\linewidth]{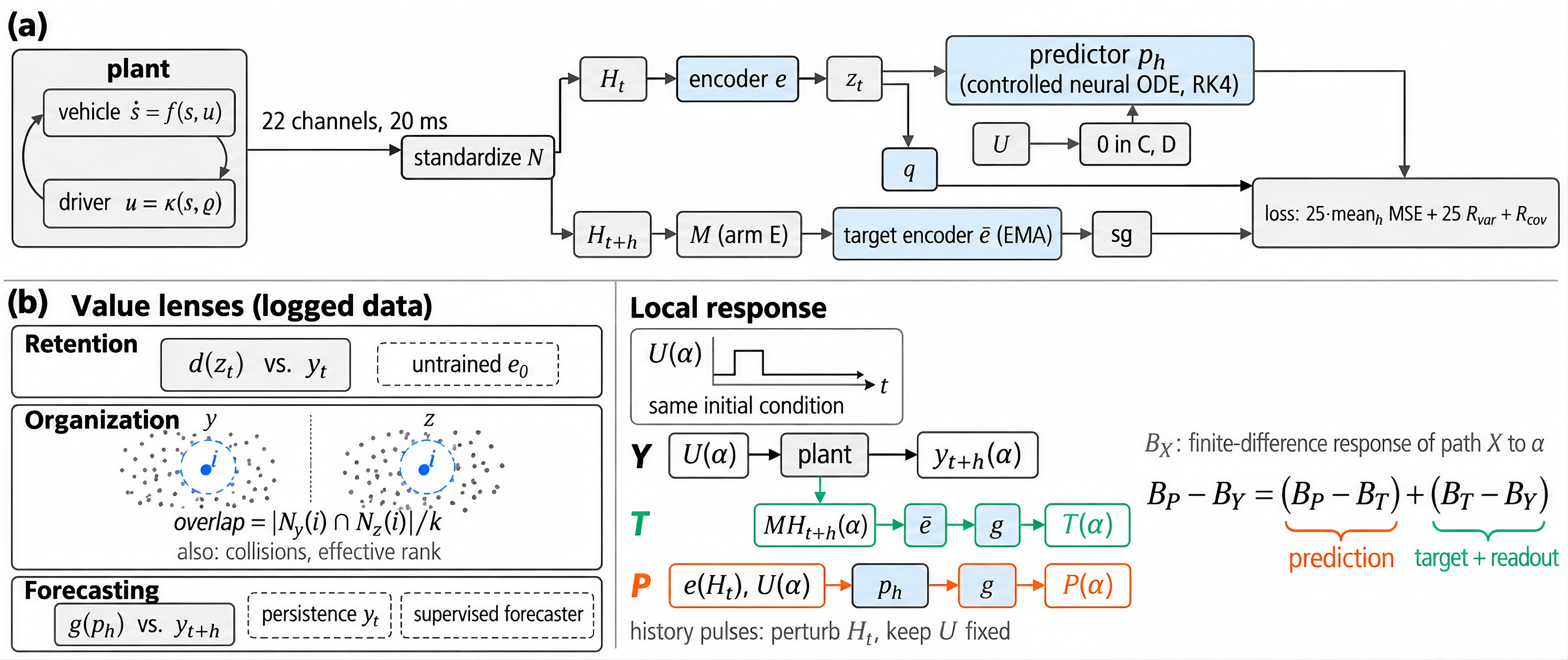}}{%
\fbox{\parbox[c][2.25in][c]{0.97\linewidth}{\centering\textbf{Figure placeholder} (full text width, about 2.3\,in tall)\\[4pt]
(a) Case study: plant and driver, recorded channels, encoder $\enc$, controlled predictor $\pred_h$, target encoder $\tenc$ with mask $M$, loss.\\
(b) Measurement protocol: value lenses and the three response paths $Y$, $T$, $P$.}}}
\caption{(a) Case study. The plant and its driver generate the recorded channels. The online encoder $\enc$ and the controlled predictor $\pred_h$ forecast the target encoder's embedding of the masked future window, and switches mark the masked predictor commands of arms C and D and the target mask $M$ of arm E. (b) Measurement protocol for any encoder, target encoder, command-conditioned predictor, and frozen readouts $d$ and $g$. Value lenses use logged data, and forecasting establishes predictive competence before responses are read. The response lens applies matched command pulses from a common initial condition and compares the plant output ($Y$), the decoded target embedding of the realized future window ($T$), and the decoded prediction ($P$), and history pulses perturb past commands, which changes $H_t$ and the plant state.}
\label{fig:overview}
\end{figure}

\textbf{Plant.}
Episodes are generated with IPG CarMaker $15.1$ \citep{ipg2024carmaker}, whose vehicle model couples the sprung body with suspension kinematics and compliance, tires, steering, hydraulic brakes, powertrain, and aerodynamics. We write the simulated vehicle and its automated driver (the plant) as $\dot s=f(s,u),$ and $ u=\kappa(s,\varrho),$ where $s$ is the full simulator state, including tire, suspension, actuator, and powertrain variables \citep{pacejka2012tyre}, $u$ contains the driver commands, and the driver model $\kappa$ follows the route $\varrho$ at a requested speed, so the commands depend on the vehicle state, the closed-loop condition behind the third motivation in Section~\ref{sec:intro}. Six configurations span front-, rear-, and all-wheel drive, combustion, hybrid, and battery-electric powertrains, from passenger cars to an open-wheel racing car (Table~\ref{tab:vehicles}). Every measurement uses this plant as its reference, so the findings describe agreement with CarMaker's vehicle model, and their transfer to physical vehicles depends on the simulator's fidelity.

\textbf{Recorded signals.}
Every $20$\,ms the simulator records $22$ channels $o_k=(y_k,u_k,a_k)$, with axes following ISO 8855 \citep{iso8855}. The reference outputs $y=(X,Y,\psi,v_x,v_y,r)$ are the position of the vehicle-frame origin, which CarMaker places at road level below the rearmost point of the vehicle, the yaw angle, the longitudinal and lateral velocities of the body's center of mass in vehicle axes, and the yaw rate. Position and yaw angle are expressed in earth-fixed axes that coincide with the vehicle frame at the start of each run, and the yaw angle is tracked continuously, growing by $2\pi$ with each full turn. The commands $u$ are the steering-wheel angle in radians and the accelerator- and brake-pedal activities in $[0,1]$. The $13$ auxiliary channels $a$ (Table~\ref{tab:channels}) are signals that vehicles with electronic stability control mostly measure, whereas lateral velocity and global position are estimated on real vehicles \citep{rajamani2012vehicle}. Channels are standardized with training statistics into $\bar o_k$, and the history window $H_t=(\bar o_{t-24},\dots,\bar o_t)$ covers $0.5$\,s.

\textbf{Model and loss.}
An online encoder maps $H_t$ to $z_t=\enc(H_t)\in\R^{64}$. A target encoder $\tenc$, an exponential moving average of $\enc$, encodes the future window $MH_{t+h}$, where the mask $M$ sets the command channels to zero, their training mean, in one training condition and is the identity otherwise. The predictor $\pred_h(z_t,U_{t:t+h})=z_t+\int_t^{t+h} f_\theta(z(\tau),u(\tau))\,d\tau$ integrates a learned controlled vector field under the future commands, with RK4 steps of $0.02$\,s and each command held over its step (Figure~\ref{fig:overview}a). Training minimizes the loss equation $\mathcal{L} = {}\frac{25}{2}\sum_{h\in\{0.5,1.0\}} \operatorname{MSE}\!\left(\pred_h(\enc(H_t),U_{t:t+h}),\mathrm{sg}\,\tenc(MH_{t+h})\right)+25\,\mathcal{R}_{\mathrm{var}}\big(q(\enc(H_t))\big)+\mathcal{R}_{\mathrm{cov}}\big(q(\enc(H_t))\big),$ where $\mathrm{sg}$ stops gradients, $q$ is a training projector, $\mathcal{R}_{\mathrm{var}}$ and $\mathcal{R}_{\mathrm{cov}}$ are the VICReg-style variance and covariance penalties \citep{bardes2022vicreg}, and the MSE averages over batch and latent coordinates, so the prediction weight $25$ applies to the mean over the two horizons. Both encoders are restored from the epoch with the lowest validation value of the loss equation, and all settings are shared by every fit (Appendix~\ref{app:implementation}).

\textbf{From the plant to the loss.}
With $\Phi_\Delta$ integrating the plant over one $20$\,ms step, $O$ recording the channels, $N$ standardizing them, and $W_t$ extracting the window ending at $t$, a trajectory $s_{k+1}=\Phi_\Delta(s_k,\kappa(s_k,\varrho))$ yields $\bar o=N\circ O(s,u)$, and each prediction term compares $\pred_h(\enc(W_t(\bar o)),U_{t:t+h})$ with $\tenc(M\,W_{t+h}(\bar o))$. The plant enters the loss only through the recorded channels, and no term reconstructs $y$.

\textbf{Readouts and training conditions.}
After training, a source readout $d$ maps $z_t$ to $y_t$, and a target readout $g$ maps target embeddings to $y_{t+h}$ and is applied unchanged to predicted embeddings, so prediction errors reach the decoded output directly \citep{ulmen2025state}. Each vehicle drives one closed route of length $2597.29$\,m in two data collections, each a schedule of start positions and requested speeds that changes between episodes. The \emph{varied-start/speed} collection uses different starts and generally higher speeds than the \emph{nominal} one (Table~\ref{tab:episodes}), so their contrast changes the commands and the speeds and track sections visited together. Five training conditions, or arms, combine the collections with two masks. Arms A and B use the nominal and varied collections with all commands, arms C and D repeat them with masked predictor commands, and arm E repeats B with masked target commands (Table~\ref{tab:arms}). Masked commands are zero in standardized units, and the history windows keep their commands in every arm. Each arm has six vehicles and three seeds, giving $90$ fits.

\section{Measurement Protocol}\label{sec:metrics}

Figure~\ref{fig:overview}b summarizes the lenses, and Appendix~\ref{app:measure} derives every reported quantity by explicit formulas.

\textbf{What prediction constrains.}\label{sec:theory}
Five idealized statements, proved in Appendices~\ref{app:setup}--\ref{app:coords}, assume a predictor that is $L$-Lipschitz in its latent argument, with error at most $\varepsilon$ on the supported commands, a uniform bound stronger than a small mean test loss, and determine how each measurement can be read. (i) Conditions whose supported future targets differ by more than $2\varepsilon$ are at least that difference minus $2\varepsilon$, divided by $L$, apart in latent space, so latent proximity constrains physical distance only where the targets are informative about the outputs. (ii) Along history pulses, the rank of the target response is bounded by that of the encoder's response, while future pulses act on the predictor alone. (iii) A uniformly small prediction error places no bound on the derivative with respect to the command, as $f_n(x)=n^{-1}\sin(n^2x)\to 0$ with $f_n'(0)=n$ shows. (iv) An invertible change of latent coordinates, compensated in the predictor and the readout, leaves the decoded forecast and its derivatives unchanged. (v) An exact predictor may merge conditions with identical supported futures, so an effective rank below six is by itself no evidence of lost information. Organization is therefore measured directly, local response separately from forecast accuracy, and responses in physical units.

\textbf{Predictive competence.}
Latent prediction is compared with target-encoder persistence through the ratio equation $\rho_h = \left( \sum_i \norm{\pred_h(\enc(H_i), U_i) - \tenc(MH_{i,h})}^2 \right) \Big/ \left( \sum_i \norm{\tenc(MH_i) - \tenc(MH_{i,h})}^2 \right),$ where $H_i$ and $H_{i,h}$ are the history and future windows of sample $i$. A fit with $\rho_h<1$ counts as a latent win, computed within each model because each fit has its own latent coordinates. A fit whose decoded forecast $g(\pred_h(z_t,U_{t:t+h}))$ has lower error than physical persistence $\hat y_{t+h}=y_t$ counts as a physical win, and decoding the true target, $g(\tenc(MH_{t+h}))$, gives the error of the readout stage alone. A supervised forecaster with the readout's hidden width and budget maps the flattened history and the future commands to the six outputs and measures how predictable the recorded futures are from these inputs.

\textbf{Retention and organization.}
Linear (ridge) and nonlinear source readouts are fitted on training windows, selected on validation windows, and frozen for testing. With one architecture and a $350$-epoch budget across representations, their $R^2$ measures accessibility to a fixed-capacity readout \citep{hewitt2019control,belinkov2022probing}. An untrained encoder with the same architecture and seed, with readouts fitted by the same recipe and the same neighborhoods, separates what the architecture passes through from what predictive training adds. For organization, outputs and full $64$-dimensional latents are standardized with training statistics, and distances are Euclidean. For a test query $i$, let $N_y(i)$ and $N_z(i)$ be its $k=30$ nearest training references in output and latent coordinates, drawn from other episodes. Overlap is $|N_y(i)\cap N_z(i)|/k$ \citep{venna2001neighborhood}, and shuffling the latent--output correspondence gives a null. A collision flags a closest latent reference beyond the query's $90$th percentile of output distances. Standardized covariance eigenvalues $\lambda_j$ give the effective rank $r_{\mathrm{eff}}=\exp(-\sum_j p_j\log p_j)$, $p_j=\lambda_j/\sum_l\lambda_l$ \citep{roy2007effective}, at most $6$ for the outputs and $64$ for the latent.

\textbf{Three paths for local response.}
For a pulse family $\alpha\mapsto U(\alpha)$ applied from a common initial condition, the three paths are $Y_h(\alpha) = y_{t+h}(\alpha)$, \quad $T_h(\alpha) = g(\tenc(MH_{t+h}(\alpha)))$, \quad $P_h(\alpha) = g(\pred_h(\enc(H_t), U(\alpha))).$ $Y$ is the plant's recorded output, $T$ decodes the target embedding of the realized future window, and $P$ decodes the predicted embedding. For history pulses the pulse alters the recorded history, and $P$ uses $\enc(H_t(\alpha))$ with the future commands fixed. The same finite-difference stencil gives the response $B$ of each path, a directional derivative with respect to the pulse parameter in units of output and command training standard deviations (Appendix~\ref{app:measure}). A direction is accepted when its estimates at pulse sizes $0.01$, $0.02$, and $0.04$ differ by less than $15\%$, which checks the linear regime, its norm exceeds $10^{-6}$, and its signal-to-noise ratio over repeated runs is at least $5$, a test that removes no nonzero, scale-consistent response because repeats agree to within $3\times10^{-14}$. For each accepted direction we report the relative error $\lVert B_P-B_Y\rVert/\lVert B_Y\rVert$, the cosine between $B_P$ and $B_Y$, and the gain $\lVert B_P\rVert/\lVert B_Y\rVert$, and the same for $T$. The identity $B_P-B_Y=(B_P-B_T)+(B_T-B_Y)$ (Appendix~\ref{app:attribution}) assigns a discrepancy to prediction or to the target and readout, and its second term combines representation content with readout capacity. The predicted latent response $D_\alpha\pred_h$ is also compared with the realized target-latent response $D_\alpha\tenc(MH_{t+h}(\alpha))$ in training-standardized target coordinates, which locates a discrepancy before the readout. To test whether response errors change a decision, each accepted future-pulse case defines $5$ to $7$ candidate commands, the template and its pulses. The model chooses the candidate with the lowest decoded cost with respect to the endpoint of a separate free-driver reference run, and normalized regret places the simulator cost of that choice between the best and the worst candidate, with keeping the template as the reference (Appendix~\ref{app:measure}).

\section{Results}\label{sec:results}\label{sec:design}

\textbf{Setup and aggregation.}
Each collection has eight $300$-second episodes per vehicle, split $4/2/2$ into training, validation, and test episodes, and windows never cross a split. Test episodes play no role in the selection of encoders or readouts. Arms are compared on a common panel of the three test settings that appear in no training episode of either collection (Table~\ref{tab:episodes}), which holds out start and speed settings with the vehicles, the route, and the driver fixed. Action contrasts are A--C and B--D, collection contrasts are B--A and D--C, and target masking is E--B. Matched vehicle and seed pairs are differenced, seeds are averaged within each vehicle, and the six vehicle effects are summarized with percentile bootstrap intervals over vehicles ($10{,}000$ draws of the mean for common-panel contrasts, $5{,}000$ of the median for local-response contrasts), so paired effects can differ from differences of the medians in Table~\ref{tab:main}. The intervals describe these six configurations, carry no multiplicity correction, and are reported with sign counts. The untrained-encoder overlap and the supervised forecaster are scored on each collection's own test episodes and paired with arms A and B (Appendix~\ref{app:implementation}).

\begin{table}[t]
\caption{Common-panel medians over $18$ fits per arm. Source $R^2$ is the current-output readout, and \emph{untr.} is a nonlinear readout fitted by the same recipe on an untrained encoder (shared by arms A/C and by B/D/E). Forecast is the one-second decoded prediction. Latent and physical wins count fits that beat target-encoder and physical persistence (Section~\ref{sec:metrics}).}
\label{tab:main}
\centering
\small
\begin{tabular}{lcccccccc}
\toprule
& \multicolumn{3}{c}{\textbf{Source $R^2$}} & \textbf{$1$\,s} & \textbf{Neighbor} & \multicolumn{2}{c}{\textbf{Latent wins}} & \textbf{Phys.\ wins} \\
\textbf{Arm} & lin. & nonlin. & untr. & \textbf{fcst.} & \textbf{overlap} & $0.5$\,s & $1$\,s & $1$\,s \\
\midrule
A & $0.929$ & $0.973$ & $0.988$ & $0.821$ & $0.267$ & $12/18$ & $18/18$ & $5/18$ \\
B & $0.913$ & $0.968$ & $0.984$ & $0.869$ & $0.267$ & $18/18$ & $18/18$ & $7/18$ \\
C & $0.931$ & $0.974$ & $0.988$ & $0.819$ & $0.266$ & $13/18$ & $18/18$ & $2/18$ \\
D & $0.913$ & $0.972$ & $0.984$ & $0.757$ & $0.264$ & $13/18$ & $18/18$ & $5/18$ \\
E & $0.931$ & $0.974$ & $0.984$ & $0.887$ & $0.313$ & $18/18$ & $18/18$ & $11/18$ \\
\bottomrule
\end{tabular}
\end{table}

\subsection{Retention, organization, and forecasting}

\textbf{Predictive competence.}
On the common panel, $74/90$ fits beat target-encoder persistence at $0.5$\,s and $90/90$ at $1$\,s. Decoded forecasts reach median macro $R^2$ of $0.868$ and $0.858$ at $0.5$ and $1$\,s, against $0.956$ and $0.885$ for persistence of the measured outputs, with $0/90$ and $30/90$ paired wins. Decoding the true target gives median $R^2$ of about $0.981$, which places most of the gap in prediction. The supervised forecaster reaches $0.992$ at $0.5$\,s and $0.985$ at $1$\,s and exceeds the matched seed-$17$ decoded forecasts ($0.921$ and $0.895$) in $12/12$ vehicle and collection cases at both horizons, so the recorded futures are predictable from the model's inputs and the gap lies in the learned predictive path.

\textbf{Retention.}
Frozen readouts recover the planar outputs in every fit, with median macro $R^2$ from $0.913$ to $0.931$ for linear readouts and from $0.968$ to $0.974$ for nonlinear ones. An untrained encoder of the same architecture, with a readout fitted by the same recipe on the same panel, reaches $0.984$ to $0.988$ and scores higher than the trained encoder in $89/90$ paired fits, against $36$ distinct references. Larger readouts narrow the median paired gap from $0.015$ to $0.005$ in the seed-$17$ pairs of arms A and B, and the untrained encoder stays higher in all $12$ (Appendix~\ref{app:extra}). Recovery is highest for longitudinal velocity and yaw rate and lowest for position (Table~\ref{tab:perout}), and with one route and one driver, position recovery may partly reflect route regularities.

\textbf{Organization, collisions, and variance.}
Latent neighborhoods keep physical organization, with median overlap $0.270$ against a null of $0.008$. The untrained encoder again scores higher, $0.304$ against $0.260$ on the own-collection test episodes of arms A and B, in $36/36$ pairs (vehicle-mean difference $-0.050$ $[-0.058,-0.041]$), and on the common panel for $k=10$, $30$, and $100$ in all seed-$17$ fits of arms A--D. Collisions are rare, with a median fraction of zero for trained and untrained encoders at the $80$th, $90$th, and $95$th percentiles, so this lens flags only large discrepancies here. Standardized effective rank is $2.690$ for the latent and $5.087$ for the outputs on the same windows, while the same latents support nonlinear readouts above $0.96$ macro $R^2$, as statement (v) of Section~\ref{sec:theory} allows.

\textbf{Action conditioning and collection.}
Under the varied collection, giving the predictor the future commands (B--D) improves the one-second decoded forecast by $0.067$ macro $R^2$ $[0.049,0.084]$, positive in $6/6$ vehicles, and changes source decoding by $-0.003$, so command availability changes the predictive path and leaves retention unchanged. Under the nominal collection (A--C) the effect is $-0.007$ $[-0.042,0.027]$, so the benefit of commands depends on the collection. The collection has no consistent effect, $0.020$ $[-0.026,0.063]$ with commands (B--A) and $-0.054$ $[-0.085,-0.023]$ without (D--C, negative in $6/6$), and since it changes both commands and operating conditions, this cannot be attributed to input variation alone.

\textbf{Target command information.}
Masking the command channels in the targets (E--B) improves neighbor overlap by $0.052$ $[0.037,0.064]$, positive in $6/6$ vehicles and close to the untrained encoder's level ($0.317$ against $0.328$ for seed $17$), and the forecast by $0.042$ $[0.010,0.081]$, positive in $5/6$, while source decoding changes by $+0.006$. With data and architecture fixed, target content thus shapes organization and forecasting, although the histories still contain the commands. At $1$\,s, arm E beats physical persistence in all six vehicles during braking and deceleration, and arm B in four of six (Appendix~\ref{app:extra}).

\subsection{Local response of the predictive path}\label{sec:response}

\textbf{Response panel.}
Exact model derivatives show that the action-conditioned arms respond to future commands, the action-masked arms give zero response, and a pulse in the second half-second leaves the $0.5$-second forecast unchanged in $90/90$ fits. Agreement with the vehicle is measured on one new scenario per vehicle on the training route (start $2450$\,m, requested speed $195$\,km/h, source endpoint at $20$\,s), with commands on an affine template between the endpoints of a free-driver reference run. Future and history pulses occupy two half-second intervals each. The $900$ matched runs realize their commands to $2\times10^{-7}$ command standard deviations. Of $144$ direction and horizon cases, $87$ are accepted, $18$ late future pulses give the expected zero response at $0.5$\,s, and the other $39$, pedal pulses below the resolution floor in two vehicles and steering or gas responses that change by more than $15\%$ across pulse sizes, provide no reference derivative (Table~\ref{tab:coverage}). Because the affine template differs from free driving, values are checked first, and arm E beats local latent persistence in $18/18$ fits and physical persistence in $13/18$ and $12/18$ at $1$\,s (Table~\ref{tab:localval}).

\textbf{Agreement with the vehicle.}
Table~\ref{tab:agree} reports one-second agreement, averaged over seeds, then the median over accepted directions and over vehicles. Arms A, B, and E respond to commands with magnitudes and directions that match the vehicle only partially, and their history-pulse responses are five to nine times larger than the vehicle's. The target and readout path aligns better for history pulses, and its future-pulse error is also large, so an accurate value readout can still transmit responses inaccurately. Masking target commands (E--B) reduces the median predictor response error in all six vehicles, by $0.013$ for future pulses and $2.684$ for history pulses, and improves the future-response cosine in $5/6$ vehicles by a median of $0.189$, with an interval that includes zero. Command availability ($3/6$ vehicles improve under varied data) and the collection give no consistent change. These conclusions hold where both learned paths pass the scale check, where the model beats both local persistence references, and for scale thresholds of $10$ to $30\%$ and norm floors of $10^{-7}$ to $10^{-5}$. At a second anchor on the same route ($2250$\,m, $180$\,km/h), with future pulses at one amplitude, the predictors of arms A, B, and E again under-respond (gains $0.15$ to $0.42$, cosines $0.10$ to $0.26$), while the target path over-responds (Appendix~\ref{app:extra}).

\begin{table}[t]
\caption{One-second vehicle-balanced three-path response agreement. Relative error and gain are with respect to the simulator path. Arms C and D have zero future-pulse response by construction, so only their history rows are shown.}
\label{tab:agree}
\centering
\small
\begin{tabular}{llcccc}
\toprule
\textbf{Arm} & \textbf{Family} & \textbf{Tgt/read.\ rel.\ err.} & \textbf{Pred/read.\ rel.\ err.} & \textbf{Pred.\ cosine} & \textbf{Pred.\ gain} \\
\midrule
A & future & $1.878$ & $1.028$ & $0.172$ & $0.320$ \\
A & history & $0.510$ & $6.785$ & $0.075$ & $6.723$ \\
B & future & $2.202$ & $1.022$ & $0.197$ & $0.472$ \\
B & history & $0.623$ & $8.193$ & $0.057$ & $8.154$ \\
C & history & $0.536$ & $7.007$ & $0.054$ & $6.769$ \\
D & history & $0.561$ & $8.761$ & $0.066$ & $9.001$ \\
E & future & $1.671$ & $0.995$ & $0.461$ & $0.495$ \\
E & history & $0.489$ & $5.103$ & $0.198$ & $5.136$ \\
\bottomrule
\end{tabular}
\end{table}

\textbf{The discrepancy is present before the readout.}
Without the readout (Table~\ref{tab:latent}), the predictor of arms B and E under-responds to future-command pulses, with a response norm of $5$ to $17\%$ of the realized one, and over-responds to history pulses by a factor of $5$ to $8$, with cosines between $0.26$ and $0.39$. Part of the mismatch therefore arises in the predictive path itself. For history pulses the predictor and readout error exceeds the target and readout error in every output and vehicle (arm E), and for future pulses both stages contribute.

\textbf{A supervised forecaster matches the direction but not the size.}
On the same pulses, the supervised forecaster of arms A and B (seed $17$, one-second $R^2$ of $0.985$) follows the direction of the vehicle's future-pulse response (cosine $0.76$, against $0.05$ and $0.15$ for the matched predictors) and overestimates its size, and its history-pulse errors are about a third of the predictors' (Table~\ref{tab:supervised}). Accurate values thus leave the response size inaccurate even without a latent, and the latent predictive path adds direction error and history over-response.

\textbf{Response errors change a local decision.}
Forty future-pulse cases pass the acceptance criteria and separate the candidates' simulator costs by more than five repeat standard deviations. Choosing with the decoded predicted cost gives vehicle-mean normalized regret of $0.549$ (arm B) and $0.575$ (arm E), against $0.194$ for keeping the template, and the model's regret is higher in all $12$ comparisons (six vehicles, arms B and E). With absolute offsets removed, regret is $0.092$ and $0.079$ for the target and readout path and $0.373$ for the predictor and readout path in both arms, so the predictive path adds decision error beyond that of the target representation. This finite-candidate test measures a local choice among commands, one component of what model-predictive control requires.

\textbf{Training the predictor on responses trades against forecasting.}
In an exploratory analysis, updating only the predictor on the responses at the tuning anchor lowers response error at the second anchor in $12/12$ fits and raises one-second forecast RMSE on the original test episodes in $12/12$ fits (Appendix~\ref{app:extra}). Under these fixed-budget updates, local response and forecast accuracy compete as training targets.

\section{Discussion and Conclusion}\label{sec:discussion}\label{sec:conclusion}

\textbf{Retention comes from the inputs, and training shapes the predictive path.}
An untrained encoder retains the observed outputs and their neighborhoods slightly better than the trained ones, command availability changes the forecast while retention stays fixed, and a supervised forecaster predicts the recorded futures from the same inputs. Retention of observed channels is therefore a property of the inputs and the architecture, which predictive training preserves without adding to it, and forecast and response errors point to the predictive path. A retention score becomes informative next to an untrained reference whenever the evaluated quantities are among the inputs, as they usually are for logged vehicle signals. Statement (i) of Section~\ref{sec:theory} is consistent with this pattern, since it constrains what an accurate predictor may merge and leaves open how the retained information is used (Appendix~\ref{app:separation}). With the outputs observed, a supervised forecaster is the direct alternative for prediction, and the protocol addresses latent models whose physical quantities are reachable only through a readout.

\textbf{Target content shapes organization and response.}
Masking commands in the prediction targets improves neighborhood overlap, one-second forecasts, forecasts during braking and deceleration, and the predictor's history-pulse response. The common direction of these four measurements suggests that command channels in the targets let the predictor match its targets partly through the commands themselves, at the expense of how the vehicle state evolves. Linear probes are consistent with this reading, since future commands explain more of arm B's target latent than of arm E's in all six vehicles, although they measure content and not use (Appendix~\ref{app:extra}). The results extend the finding of \citet{tan2026know} that the target decides which physical parameters a latent acquires to organization and local response.

\textbf{Forecast accuracy and local response separate.}
Arm E beats both local persistence references at the response anchor in most fits, yet its responses are too small for future commands and too large for history changes, already in latent coordinates, and these errors raise regret among nearby commands, a pattern that matches the third motivation of our hypothesis. In logged driving, commands vary together with the vehicle state, so a predictor can fit its targets through the history and learn little about the separate effect of a command \citep{forssell1999closed}, and under-response to future commands with over-response to history changes is consistent with such a predictor. The supervised forecaster, trained on the same driving, follows the direction of the vehicle's response but not its size, so part of the response error follows from the data and part from the latent path. A planner that ranks or differentiates predictions across commands relies on this property, which forecast scores leave unmeasured, and should test it along the directions it will use.

\textbf{Limitations.}
The episodes come from one route and one automated driver, the held-out settings keep the route, the driver, and the vehicles fixed, and the bootstrap intervals describe six selected vehicles. One hyperparameter protocol and one readout budget are shared across vehicles and arms. The references are an untrained encoder and a supervised forecaster, without a calibrated physics model or a reconstruction-trained encoder, and the simulated signals are noise-free. The response study uses one scenario per vehicle with an affine command template, and the ranking test and the predictor updates use single anchors and seed-$17$ fits.

\textbf{Conclusion.}
We asked what an action-conditioned latent predictive representation of a vehicle retains, how it organizes that information, and whether its predictions change with commands as the vehicle does, and the protocol answers each question separately. In the CarMaker case study the answers support our hypothesis. Retention and organization match an untrained encoder of the same architecture, forecasting and local response depend on the predictive path and on the content of its targets, and models that beat persistence still respond to commands differently from the vehicle, enough to change which command they would choose. Command sequences designed independently of the driver, other routes and drivers, physics-based references, and real vehicles would show how far these findings extend. For any latent model proposed for control, the protocol identifies which of these properties a controller can rely on, and it provides the measurements against which closed-loop results can be interpreted.

\subsection*{Ethics statement}
This work uses simulated vehicle trajectories generated in IPG CarMaker. No human subjects, personal data, or field experiments are involved. The study measures representation content and makes no claim of closed-loop safety or control sufficiency. Model responses and representation measurements must not be interpreted as safety or controller-performance certificates.
 
\subsection*{Reproducibility statement}
Sections~\ref{sec:model} and~\ref{sec:results} state the model, loss, horizons, collections, arms, splits, and evaluation protocol. Section~\ref{sec:response} and Appendix~\ref{app:implementation} record the simulator-timed pulses, numerical gate, coverage, and binary provenance. Section~\ref{sec:metrics} defines all metrics. The derivations are in Appendices~\ref{app:setup}--\ref{app:coords}. An anonymized release of code, checkpoints, per-fit tables, response traces, and figure-generation scripts will be available publicly in a final version of the paper.
 
\subsection*{AI use statement}
AI tools (Grammarly, and Claude) assisted manuscript text review (cohesion and orthography), figure generation (Fig. 1 only, ChatGPT), and code review. The authors are responsible for the experiments, numerical results, citations, and final text.


\bibliography{references}
\bibliographystyle{iclr2027_conference}

\appendix

\section{Setup and Assumptions}\label{app:setup}

The statements in Section~\ref{sec:theory} are idealized, local, and relative to the supported control distribution. Let $\mathcal{X}$ be the set of physical conditions reachable under the data-generating protocol and $x\in\mathcal{X}$ the condition at an anchor time, which includes the plant state $s$ of Section~\ref{sec:model}, with reference outputs $y(x)=(X,Y,\psi,v_x,v_y,r)$. The simulator carries further internal variables, and $y$ records six of them. Let $\Uset$ be the supported future control sequences of duration $h$ and $F_U$ the map induced by applying $U\in\Uset$. Let $H(x)$ be the signal window ending at the anchor time and $H^+(x,U)$ the window ending $h$ later. Write $\enc(x):=\enc(H(x))$ and define the target feature $\tau_U(x):=\tenc\big(H^+(x,U)\big)$.

\begin{assumption}[Supported prediction accuracy]\label{as:eps}
$\norm{\pred(\enc(x),U)-\tau_U(x)}\le\varepsilon$ for all $x\in\mathcal{X}$ and $U\in\Uset$.
\end{assumption}

\begin{assumption}[Predictor Lipschitz continuity]\label{as:lip}
$\norm{\pred(z,U)-\pred(z',U)}\le L\norm{z-z'}$ for all $z,z'$ in the image of $\enc$ and all $U\in\Uset$.
\end{assumption}

Assumption~\ref{as:eps} is a uniform bound on the training support, which is stronger than a small mean test loss. Assumption~\ref{as:lip} holds for a finite network with some constant, and the bounds below degrade linearly in it.

\section{What Prediction Constrains}\label{app:separation}

\begin{proposition}[Behavioral separation]\label{prop:sep}
Under Assumptions~\ref{as:eps} and~\ref{as:lip}, for all $x,x'\in\mathcal{X}$ and $U\in\Uset$,
\begin{equation}
\norm{\tau_U(x)-\tau_U(x')} \le L\,\norm{\enc(x)-\enc(x')} + 2\varepsilon .
\label{eq:sepapp}
\end{equation}
\end{proposition}

\begin{proof}
Insert the two predictions and apply the triangle inequality:
\begin{align*}
\norm{\tau_U(x)-\tau_U(x')}&\le\norm{\tau_U(x)-\pred(\enc(x),U)}+\norm{\pred(\enc(x),U)-\pred(\enc(x'),U)}\\
&\quad+\norm{\pred(\enc(x'),U)-\tau_U(x')}\\
&\le 2\varepsilon+L\norm{\enc(x)-\enc(x')}.
\end{align*}
\end{proof}

At $\varepsilon=0$, $\enc(x)=\enc(x')$ implies $\tau_U(x)=\tau_U(x')$ for every $U\in\Uset$, so an exact predictor keeps apart conditions with different supported targets. Taking the supremum over $U$ gives the finite-horizon behavioral pseudometric $d^\tau_h(x,x')=\sup_{U\in\Uset}\norm{\tau_U(x)-\tau_U(x')}\le L\norm{\enc(x)-\enc(x')}+2\varepsilon$. It is a pseudometric because conditions with identical supported targets have distance zero.

\paragraph{Physical recovery.}\label{app:recovery}
Suppose supported targets are locally informative about the outputs, $d^\tau_h(x,x')\ge\alpha\norm{y(x)-y(x')}$ for some $\alpha>0$ on a neighborhood. Then
\begin{equation}
\norm{\enc(x)-\enc(x')} \ge \frac{\alpha\,\norm{y(x)-y(x')}-2\varepsilon}{L},
\end{equation}
so the encoder separates output differences larger than $2\varepsilon/\alpha$. The assumption fails when conditions that differ in $y$ produce the same targets under every supported control sequence, which one route and one deterministic driver make possible. With $\varepsilon>0$, separation holds only above a finite scale, so a zero collision fraction cannot be read as injectivity.

\paragraph{Scale for the trained predictors.}
The vector field of the predictor is a spectrally normalized MLP with $\tanh$ activations, so its Lipschitz constant in $z$ at fixed command is at most the product $L_f$ of the three layer spectral norms. Computed exactly on the $90$ checkpoints, $L_f$ lies between $1.000003$ and $1.000369$. One RK4 step of length $\Delta$ is then Lipschitz with constant at most $\sum_{j=0}^{4}(\Delta L_f)^j/j!\le e^{\Delta L_f}$, so over $25$ and $50$ steps of $0.02$\,s, $L\le1.649$ at $0.5$\,s and $L\le2.718$ at $1$\,s. On $10{,}000$ random pairs of common-panel latents with shared commands per seed-$17$ fit, the largest observed expansion ratio has arm medians of at most $1.06$ at $0.5$\,s and $1.14$ at $1$\,s and stays below the bound in every fit. Taking $\varepsilon$ as the $95$th percentile of the prediction error, the separation scale $2\varepsilon/L$ is $3.3$ to $6.0$ times the median latent nearest-neighbor distance (arm medians, seed $17$). The statement is therefore consistent with the trained models, and at this scale it motivates the neighborhood measurements without replacing them.

\paragraph{Predictive equivalence.}\label{app:quotient}
Conditions with $\tau_U(x)=\tau_U(x')$ for all $U\in\Uset$ are equivalent for prediction. An encoder that separates these equivalence classes admits an exact predictor, and by the $\varepsilon=0$ case above any exact predictor requires that separation. Distinctions within a class are left to the encoder. Covariance effective rank measures the spread of sampled latents and is a separate quantity from the dimension of these classes. Enlarging $\Uset$ can refine the classes, so conclusions are relative to the studied driving protocol.

\paragraph{Reading neighborhood overlap.}
If locally
\[
\ell\norm{y(x)-y(x')}\le\norm{\enc(x)-\enc(x')}\le L_e\norm{y(x)-y(x')},
\]
latent and physical distance rankings agree up to the factor $L_e/\ell$. Partial replacement of neighbors is then expected once the neighbor radius is comparable to the scale over which this distortion reorders distances. An overlap below one is therefore consistent with an anisotropic injective encoding, and lost information has to be shown by readouts or collisions.

\section{Local Response: What the Identities Bound}\label{app:rank}

\paragraph{History pulses.}
Let $\beta\in\R^{k}$ parameterize simulator-realizable past command pulses with future commands $U$ fixed, inducing histories $H(\beta)$ and latents $Z(\beta)=\enc(H(\beta))$. If prediction is exact and differentiable along this family, $\tau_h(\beta)=\pred_h(Z(\beta),U)$, and the chain rule gives
\begin{equation}
D_\beta\tau_h = D_z\pred_h\,D_\beta Z, \qquad \operatorname{rank} D_\beta\tau_h\le\operatorname{rank} D_\beta Z,
\end{equation}
because the rank of a product is at most the rank of either factor. Target responses that vary in $d$ independent pulse directions therefore require the encoder to vary in at least $d$ directions along the same family.

\paragraph{Future pulses.}
With the history fixed and $U(\alpha)$ perturbed, $D_\alpha\tau_h=D_U\pred_h\,D_\alpha U$ involves only the predictor, so future pulses bound the rank of $D_U\pred_h$ and leave encoder rank open. If future controls depend on the condition, $U=\pi(x)$, then $D_x[\pred(\enc(x),\pi(x))]=D_z\pred\,D\enc+D_U\pred\,D\pi$, and the policy term can contribute rank. Holding the command template fixed across the members of a family separates the two terms. Covariance effective rank describes a sampled distribution and is a third, separate quantity.

\paragraph{Values do not bound derivatives.}\label{app:valuederiv}
The example of Section~\ref{sec:theory} extends to predictors. The predictor $\tilde\pred(z,U)=\pred(z,U)+\varepsilon\sin(\omega c^\top U)w$ with unit vectors $c,w$ differs from $\pred$ by at most $\varepsilon$ in value and by $\varepsilon\omega$ in the derivative along $c$, which is unbounded in $\omega$. Conversely, a poor value fit can have an accurate derivative in a given direction. Responses are therefore reported per direction and amplitude and accepted only after the scale and repeat checks of Section~\ref{sec:metrics}.

\paragraph{Error attribution.}\label{app:attribution}
For the three paths of Section~\ref{sec:metrics}, $J_P-J_Y=(J_P-J_T)+(J_T-J_Y)$ holds for any shared linear finite-difference stencil, before and after derivative convergence. $J_P$ involves the online encoder and predictor, and $J_T$ involves the target encoder applied to the realized future window. The second term combines target-representation content and readout capacity, and separating the two requires a readout-free comparison such as the latent responses below.

\section{Coordinate Invariance}\label{app:coords}

\begin{proposition}[Reparameterization invariance]\label{prop:coords}
Let $\pred_h(\enc(H),U)=\tenc(H^+)$ on the support, and let $T_o,T_T$ be invertible. With $\enc'=T_o\circ\enc$, $\tenc'=T_T\circ\tenc$, $\pred_h'(z,U)=T_T(\pred_h(T_o^{-1}(z),U))$, and $g'=g\circ T_T^{-1}$, exact prediction is preserved and $g'(\pred_h'(\enc'(H),U))=g(\pred_h(\enc(H),U))$.
\end{proposition}

\begin{proof}
$\pred_h'(\enc'(H),U)=T_T(\pred_h(\enc(H),U))=T_T(\tenc(H^+))=\tenc'(H^+)$, and applying $g'$ cancels $T_T$. When the maps are differentiable, derivatives of the physical forecast are equal as well.
\end{proof}

The invariance is exact at zero loss. At finite loss a reparameterization changes the metric of the residual and interacts with the network class and the variance--covariance terms of the loss in Section~\ref{sec:model}. The proposition motivates comparing responses in physical-output coordinates.

\section{Measurement Definitions}\label{app:measure}

This appendix traces every reported quantity from the plant and the model by explicit formulas. Sections~\ref{sec:model} and~\ref{sec:metrics} state the same definitions in shorter form.

\paragraph{Plant and windows.}
At each $20$\,ms step $k$ the plant records $o_k=(y_k,u_k,a_k)\in\R^{22}$: the six reference outputs $y_k=(X,Y,\psi,v_x,v_y,r)$, the three driver commands $u_k$ (steering-wheel angle and accelerator- and brake-pedal activity), and $13$ auxiliary channels $a_k$ listed in Appendix~\ref{app:implementation}. Its internal state $s_k$ also contains tire, suspension, actuator, and driver variables that are not recorded. Channels are standardized with training statistics, $\bar o_k=(o_k-\mu)\oslash\sigma$. The source window is $H_t=(\bar o_{t-24},\dots,\bar o_t)$ and the future window at horizon $h$ is $H_{t+h}$, defined in the same way. $U_{t:t+h}$ is the standardized command sequence over $(t,t+h]$, and $\sigma_y,\sigma_u$ are the training standard deviations of outputs and commands.

\paragraph{Model.}
The encoder gives $z_t=\enc(H_t)\in\R^{64}$. The predictor $\pred_h(z_t,U_{t:t+h})$ integrates $\dot z=f_\theta(z,u)$ from $z_t$ with RK4 steps of $0.02$\,s, holding the sampled command fixed within each step ($25$ steps for $h=0.5$\,s, $50$ for $h=1$\,s). The vector field $f_\theta:\R^{64+3}\to\R^{64}$ is a spectrally normalized MLP with widths $67\to128\to128\to64$ and $\tanh$ activations between its linear layers. The target is $\bar z_{t+h}=\tenc(MH_{t+h})$, where $M$ sets the command channels to zero, the training mean, in arm E and is the identity otherwise. In arms C and D the predictor receives $U=0$. The source readout $d$ is fitted on pairs $(z_t,y_t)$ and the target readout $g$ on pairs $(\bar z_{t+h},y_{t+h})$ from training windows, both with standardized inputs and outputs, a mean-squared-error loss, and early stopping on validation windows. The yaw angle $\psi$ is tracked continuously and taken relative to the run start. The untrained reference replaces $\enc$ by its initialization $\enc_0$ and fits its own readout $d_0$ with the same procedure. The supervised forecaster $F_\phi$ maps $(\mathrm{vec}\,H_t,U_{t:t+h})$ to $y_{t+h}$.

\paragraph{Value metrics.}
For predictions $\hat y$ of an output $j$ over evaluation windows, $R^2_j=1-\sum(\hat y_j-y_j)^2/\sum(y_j-\bar y_j)^2$, and macro $R^2$ averages the six outputs. Current-output retention uses $\hat y_t=d(z_t)$, and $d_0(\enc_0(H_t))$ for the reference. The decoded forecast is $\hat y_{t+h}=g(\pred_h(z_t,U_{t:t+h}))$, the decoding ceiling is $g(\bar z_{t+h})$, and physical persistence is $\hat y_{t+h}=y_t$. Latent skill is the ratio $\rho_h$ of Section~\ref{sec:metrics} with persistence $\tenc(MH_t)$. For local value competence, $\mathrm{NRMSE}=\big(\tfrac16\sum_j((\hat y_j-y_j)/\sigma_{y,j})^2\big)^{1/2}$ at the zero member.

\paragraph{Pulse families.}
A future pulse on command $c$ over interval $I\in\{(0,0.5],(0.5,1]\}$\,s sets $U(q)=U_0+q\,e_c\mathbf{1}_I$, where $U_0$ is the command template, $e_c$ selects channel $c$, and $q$ is measured in units of $\sigma_{u,c}$. The source history $H_t$ is unchanged. A history pulse applies the same form to the recorded commands over $I'\in\{(-0.5,-0.26],(-0.26,-0.02]\}$\,s with the future template $U_0$ fixed, and the plant then changes $y$, $a$, and the hidden state $s$ inside $H_t$. Each member $q$ is a separate simulator run from the same initial condition.

\paragraph{Three paths and their derivatives.}
For a future pulse the three paths are
\begin{equation}
Y(q)=y_{t+h}(q),\qquad T(q)=g\big(\tenc(MH_{t+h}(q))\big),\qquad P(q)=g\big(\pred_h(\enc(H_t),U(q))\big),
\end{equation}
and for a history pulse $P(q)=g\big(\pred_h(\enc(H_t(q)),U_0)\big)$, with $Y$ and $T$ unchanged in form. For a path $X$, the response is the finite-difference estimate
\begin{equation}
B_X=\sigma_y^{-1}\odot\sum_i w_i\,\bar X(q_i),
\end{equation}
with repeat means $\bar X$, central weights $w=(-1,1)/(2s)$ at $q=(-s,s)$ for steering, and second-order one-sided weights $w=(-3,4,-1)/(2s)$ at $q=(0,s,2s)$ for pedals, signed by the admissible direction, with $s\in\{0.01,0.02,0.04\}$. The same weights give the componentwise propagated standard error $\epsilon_X=\big(\sum_i w_i^2\,\widehat{\mathrm{Var}}_i/n_i\big)^{1/2}$ from the repeat variances, and the gate uses $\mathrm{SNR}=\norm{B_Y}/\max(\norm{\epsilon_Y},10^{-8})$, a repeat-based check of numerical resolution. The floor $10^{-8}$ applies when repeats agree exactly. Repeat standard errors are zero in $120$ of the $144$ physical cases and at most $2.6\times10^{-14}$ otherwise. To first order in $q$,
\begin{align}
B_Y&\approx\sigma_y^{-1}\odot D_q\,y_{t+h}, &
B_T&\approx\sigma_y^{-1}\odot Dg\,D\tenc\,M\,D_qH_{t+h},\\
B_P^{\mathrm{fut}}&\approx\sigma_y^{-1}\odot Dg\,D_U\pred_h\,e_c\mathbf{1}_I, &
B_P^{\mathrm{hist}}&\approx\sigma_y^{-1}\odot Dg\,D_z\pred_h\,D\enc\,D_qH_t .
\end{align}
These expressions make three properties explicit. First, $D_qH_{t+h}$ contains the standardized changes of the recorded future outputs, including $D_q y_{t+h}$ at the window end, so $T$ re-encodes the plant's own response and $J_T-J_Y$ measures how faithfully $\tenc$ and $g$ transmit it. Second, the future-pulse predicted response involves no encoder derivative, and the history-pulse response passes through $D\enc$ applied to changes in the recorded outputs and commands. Third, $B_Y$ is a finite-horizon sensitivity of recorded outputs to pulse parameters, distinct from a state Jacobian $\partial s_{k+1}/\partial s_k$, and every comparison is made in output coordinates. Exact model derivatives replace the finite difference in $B_P^{\mathrm{fut}}$ by automatic differentiation through $g$ and the RK4 steps of $\pred_h$.

\paragraph{Response metrics.}
With $J_X=B_X$ for one direction and horizon, we report
\[
\frac{\norm{J_X-J_Y}}{\norm{J_Y}},\qquad \frac{J_X^\top J_Y}{\norm{J_X}\norm{J_Y}},\qquad \frac{\norm{J_X}}{\norm{J_Y}},
\]
the relative error, cosine, and gain. The cosine is left undefined when either norm is below $10^{-10}$. A direction is accepted when the three estimates at adjacent step sizes differ by less than $15\%$, the SNR of $J_Y$ is at least $5$, and $\norm{J_Y}>10^{-6}$. Vehicle-balanced summaries average seeds within a case, take the median over accepted directions within a vehicle, and report the median of vehicles.

\paragraph{Latent responses.}
In target coordinates standardized by $S=\mathrm{diag}\big(\max(\bar\sigma_i,0.1\,\mathrm{med}_l\,\bar\sigma_l)\big)$, with $\bar\sigma$ the training standard deviations of target coordinates, the realized and predicted latent responses are $L_T=S^{-1}\sum_iw_i\,\tenc(MH_{t+h}(q_i))$ and $L_P=S^{-1}\sum_iw_i\,\pred_h(\cdot,\cdot)(q_i)$ with the arguments of $P$. We report $\norm{L_P}/\norm{L_T}$, $\norm{L_P-L_T}/\norm{L_T}$, and their cosine, and require $L_T$ to pass the same gate.

\paragraph{Candidate ranking.}
For an accepted future-pulse case with candidates $q_1,\dots,q_m$ (including $q=0$), the cost of an output vector is $C(y)=\sum_{j\in\{Y,\psi,v_x,r\}}\big((y_j-y^\star_j)/\sigma_{y,j}\big)^2$, where $y^\star$ is the endpoint of a separate free-driver reference run. The model chooses $\hat k=\arg\min_k C(P(q_k))$, and normalized regret is $\big(C(Y(q_{\hat k}))-\min_kC(Y(q_k))\big)/\big(\max_kC(Y(q_k))-\min_kC(Y(q_k))\big)$, compared with the regret of $q=0$. The offset-removed diagnostic uses $\tilde P(q_k)=Y(0)+P(q_k)-P(0)$, and the same for $T$.

\paragraph{Geometry and paired effects.}
Overlap, the correspondence null, the collision fraction, and effective rank are defined in Section~\ref{sec:metrics}. For a contrast between arms $a$ and $b$, the vehicle effect is $\Delta_v=\tfrac13\sum_{\text{seeds}}(m_a-m_b)$ for a metric $m$. All six outputs and all $64$ latent coordinates are z-scored with training-reference statistics, the yaw angle entering as the continuous angle relative to the run start, and neighbors use Euclidean distance on $4096$ reference and $512$ query windows sampled deterministically per fit. Effective rank uses the squared singular values of the centered, standardized test values, so the physical and latent medians $5.087$ and $2.690$ correspond to $0.85$ and $0.04$ of their maxima. Overlapping windows within an episode are treated as dependent, so intervals are formed over vehicles. Common-panel contrasts report the mean of the six $\Delta_v$ with a $95\%$ percentile bootstrap over vehicles with $10{,}000$ draws, and local-response contrasts report the median vehicle effect with a percentile bootstrap of the median and $5{,}000$ draws.

\paragraph{Scope of the protocol.}
Retention and geometry require recorded reference outputs. The target path requires an encoder of realized future windows, and for a model without a separate target encoder the model's own encoder takes its place. Response measurement requires matched command perturbations from a common initial condition. A simulator provides them exactly. On a physical vehicle, repeated runs from matched conditions could approximate them, with the repeat variance entering the resolution gate above. Model-side derivatives require differentiable $\pred_h$ and $g$, or the same finite-difference stencil applied to the model. The definitions apply to any encoder and predictor with these interfaces, whatever their training objective.

\section{Implementation and Protocol Details}\label{app:implementation}\label{app:metrics}

\paragraph{Vehicles and collections.}
The six configurations are Audi R8 Weather, IPG CompanyCar EV, IPG CompanyCar ModularPT AxleSplit P2, MB CClass350e, McLaren F1, and DemoCar Ackermann. Each has a nominal RaceDriver collection and a varied-start/speed collection on one closed Hockenheim route of length $2597.29$\,m, each with eight $300$-second episodes split $4/2/2$ into train, validation, and test. Table~\ref{tab:episodes} lists the start position and requested speed of every episode, identical for the six vehicles, and Table~\ref{tab:arms} the five training conditions. The common panel uses the three test settings that appear in no training episode of either collection, $(1300\,\mathrm{m},180\,\mathrm{km/h})$ and $(2200\,\mathrm{m},190\,\mathrm{km/h})$ from the nominal and $(2300\,\mathrm{m},212\,\mathrm{km/h})$ from the varied test episodes. The varied test setting $(1000\,\mathrm{m},190\,\mathrm{km/h})$ repeats a nominal training setting and is left out.

\begin{table}[h]
\caption{Episode schedule per collection, as start position (m) and requested speed (km/h).}
\label{tab:episodes}
\centering
\scriptsize
\setlength{\tabcolsep}{4pt}
\begin{tabular}{lcccccccc}
\toprule
& \multicolumn{4}{c}{\textbf{Training}} & \multicolumn{2}{c}{\textbf{Validation}} & \multicolumn{2}{c}{\textbf{Test}} \\
\textbf{Collection} & 1 & 2 & 3 & 4 & 1 & 2 & 1 & 2 \\
\midrule
Nominal & 5/160 & 500/170 & 1000/190 & 1500/200 & 700/175 & 1750/185 & 1300/180 & 2200/190 \\
Varied start/speed & 80/185 & 650/195 & 1200/205 & 1850/215 & 1450/200 & 2150/210 & 1000/190 & 2300/212 \\
\bottomrule
\end{tabular}
\end{table}

\begin{table}[h]
\caption{Five matched training conditions, with three seeds ($17$, $29$, $43$) per vehicle, giving $90$ fits.}
\label{tab:arms}
\centering
\scriptsize
\begin{tabular}{llll}
\toprule
\textbf{Arm} & \textbf{Collection} & \textbf{Predictor commands} & \textbf{Target commands} \\
\midrule
A & Nominal & Available & Available \\
B & Varied start/speed & Available & Available \\
C & Nominal & Masked & Available \\
D & Varied start/speed & Masked & Available \\
E & Varied start/speed & Available & Masked \\
\bottomrule
\end{tabular}
\end{table} Seeds are $17$, $29$, and $43$. Table~\ref{tab:vehicles} lists the configurations, taken from the CarMaker $15.1$ example data sets. All use IPG RealTime tire models, hydraulic brakes, a steering-wheel angle input, and six-coefficient aerodynamics.

\begin{table}[h]
\caption{Vehicle configurations. CompanyCar denotes the IPG CompanyCar models, and McLaren F1 is an open-wheel racing car. Wheelbase is the longitudinal distance between front and rear wheel centers in the vehicle data set.}
\label{tab:vehicles}
\centering
\small
\begin{tabular}{lllcl}
\toprule
\textbf{Vehicle} & \textbf{Powertrain} & \textbf{Driven axles} & \textbf{Wheelbase} & \textbf{Tires (front/rear)} \\
\midrule
DemoCar Ackermann & Combustion & Front & $2.54$\,m & 195/65 R15 \\
Audi R8 Weather & Combustion & Both & $2.75$\,m & 245/35 R19, 295/35 R19 \\
CompanyCar EV & Battery electric & Both, per axle & $2.62$\,m & 255/35 R20, 295/30 R20 \\
CompanyCar ModularPT P2 & P2 hybrid & Both, per axle & $2.62$\,m & 255/35 R20, 295/30 R20 \\
MB CClass350e & Parallel hybrid & Rear & $2.84$\,m & 225/50 R17, 245/45 R17 \\
McLaren F1 & Combustion & Rear & $3.58$\,m & racing \\
\bottomrule
\end{tabular}
\end{table}

\paragraph{Windows.}
Signals are decimated from $10$ to $20$\,ms. Each window uses $25$ past samples ($0.5$\,s), $25$ target samples per horizon, and $50$ future command steps ($1.0$\,s), sampled with stride five and at most $900$ windows per episode. The $22$ input channels are listed in Table~\ref{tab:channels} with their CarMaker names, in ISO 8855 axes with $x$ forward, $y$ to the left, and $z$ upward in the vehicle frame. Normalization statistics are fitted on training windows.

\begin{table}[h]
\caption{Input channels, all sampled at $20$\,ms.}
\label{tab:channels}
\centering
\scriptsize
\begin{tabular}{lll}
\toprule
\textbf{Group} & \textbf{Signal} & \textbf{CarMaker quantity} \\
\midrule
Outputs ($6$) & Fr1-origin position, yaw angle & \texttt{Car.Fr1.tx}, \texttt{Car.Fr1.ty}, \texttt{Car.Yaw} \\
 & body center-of-mass velocities, yaw rate & \texttt{Car.vx}, \texttt{Car.vy}, \texttt{Car.YawRate} \\
Commands ($3$) & steering-wheel angle, pedal activities & \texttt{DM.Steer.Ang}, \texttt{DM.Gas}, \texttt{DM.Brake} \\
Auxiliary ($13$) & body accelerations & \texttt{Car.ax}, \texttt{Car.ay}, \texttt{Car.az} \\
 & speed, roll and pitch rates & \texttt{Car.v}, \texttt{Car.RollVel}, \texttt{Car.PitchVel} \\
 & road-wheel steering angle & \texttt{Steer.WhlAng} \\
 & wheel speeds & \texttt{Car.vFL}, \texttt{Car.vFR}, \texttt{Car.vRL}, \texttt{Car.vRR} \\
 & master-cylinder pressure & \texttt{Brake.Hyd.Sys.pMC} \\
 & engine speed & \texttt{PT.Engine.rotv} \\
\bottomrule
\end{tabular}
\end{table}

\paragraph{Architecture and optimization.}
The encoder is a temporal CNN with widths $(64,128,128)$ and kernels $(5,5,3)$, producing a $64$-dimensional latent. The predictor is a controlled neural ODE integrated with RK4 in $0.02$\,s steps under the future command sequence. Its vector field is a spectrally normalized $\tanh$ MLP ($67\to128\to128\to64$). The target encoder is an exponential moving average of the online encoder with coefficient $0.996$. Training uses batch size $256$, encoder and predictor learning rates $3\times 10^{-4}$ and $9\times 10^{-4}$, weight decay $10^{-4}$, and prediction, variance, and covariance weights $25/25/1$, with no reconstruction head. Training runs between $20$ and $120$ epochs, with gradient-norm clipping at $1$, and keeps the epoch with the lowest validation value of the full objective, an improvement requiring a decrease of at least $10^{-4}$ and training stopping after $15$ epochs without one. The median fit stops at $34$ epochs and one of $90$ reaches the cap. The selected online and moving-average encoders are restored together. These settings are shared across vehicles and arms.

\paragraph{Readouts.}
Readouts are fitted after representation training on training windows and selected on validation windows. Linear readouts are ridge regressions. Nonlinear readouts are MLPs with two width-$128$ GELU hidden layers, a $350$-epoch maximum, and validation patience $40$, with the same architecture and budget on every representation. Each readout is trained on training windows only, with its checkpoint selected on validation windows and no refit on validation data. Validation scores still rise slowly near this budget, so reported readout $R^2$ may understate what a larger readout recovers. The yaw angle is regressed as a continuous angle relative to the run start.

\paragraph{Reference models.}
The untrained encoder has the architecture of the trained encoder, is initialized with the same seed, and receives no representation training. Its readout is fitted by the same recipe, and its neighborhoods are computed identically. There is one untrained encoder per vehicle, collection, and seed ($36$ in total), shared by the arms trained on that collection, and its readout is scored on the common panel for every arm. The supervised forecaster is an MLP with two width-$128$ hidden layers from the flattened $0.5$-second history and the logged future command sequence to the six outputs at $0.5$ or $1$\,s, trained for at most $350$ epochs with validation selection on the same episode roles. For the pulse analysis, the seed-$17$ forecasters were retrained with the same recipe and reproduce the reported test $R^2$ within $0.005$.

\paragraph{Response acquisition.}
Simulator-side expressions fix all three commands on an affine open-loop template joining the free-reference endpoints at the model's $20$\,ms samples. An affine template keeps the expressions within the complexity limit of the simulator interface, which sample-by-sample replay exceeds. Each perturbed command value is run twice and each family has three zero runs. Steering uses central stencils, and pedals use central or second-order one-sided stencils depending on actuator bounds. The realized commands of all $900$ runs match the specified family to $1.87\times 10^{-7}$ command training standard deviations. Signals are exported in single precision, and the propagated quantization bound stays below $0.38\%$ of the response norm in all $87$ accepted cases. Response experiments use the stock CarMaker $15.1$ binary, and training data used a custom build of the same version.

\paragraph{Latent-space response and ranking.}
Latent coordinates are standardized by each fit's training target standard deviations, floored at $0.1$ of the median coordinate standard deviation. The realized latent response must pass the same gate as the physical response, with less than $15\%$ change between adjacent amplitudes, repeat-propagated SNR of at least $5$, and norm above $10^{-6}$. Reloading all $90$ checkpoints reproduces the saved decoded outputs within a maximum absolute difference of $0.0017$. For ranking, $40$ of the $72$ future-pulse direction/horizon cases pass the physical gate and separate the simulator costs of their candidates by more than five repeat standard deviations, with a $10^{-8}$ floor. Every candidate is a simulated run, and combinations of pulses on different commands are not formed.

\paragraph{Response-targeted predictor updates.}
Encoder and readout are frozen and only the predictor is updated, for $80$ fixed steps with no checkpoint or weight selection on the evaluation anchor. The tuning anchor is the response panel of Section~\ref{sec:response}. The evaluation anchor ($2250$\,m, $180$\,km/h) uses the same $0.5$-second pulse duration and an amplitude of $0.04$ training standard deviations. It adds $90$ new runs, and the $84$ pulsed runs among them realize their commands and match the zero run before the pulse. About half of the accepted cases improve under the first two losses ($52/104$ and $50/104$), mostly for steering. The physical-response loss matches decoded pulse responses to the simulator, and the latent-response loss matches the predictor's pulse-induced latent change to the realized target-latent change without the readout. The forecast-preserving variant adds a weight-one penalty on changes to the original predictor outputs over $64$ training windows. Forecast change is measured on $256$ fixed windows from the original test episodes.

\section{Additional Results}\label{app:extra}

\fig[tbp]{width=0.92\linewidth}{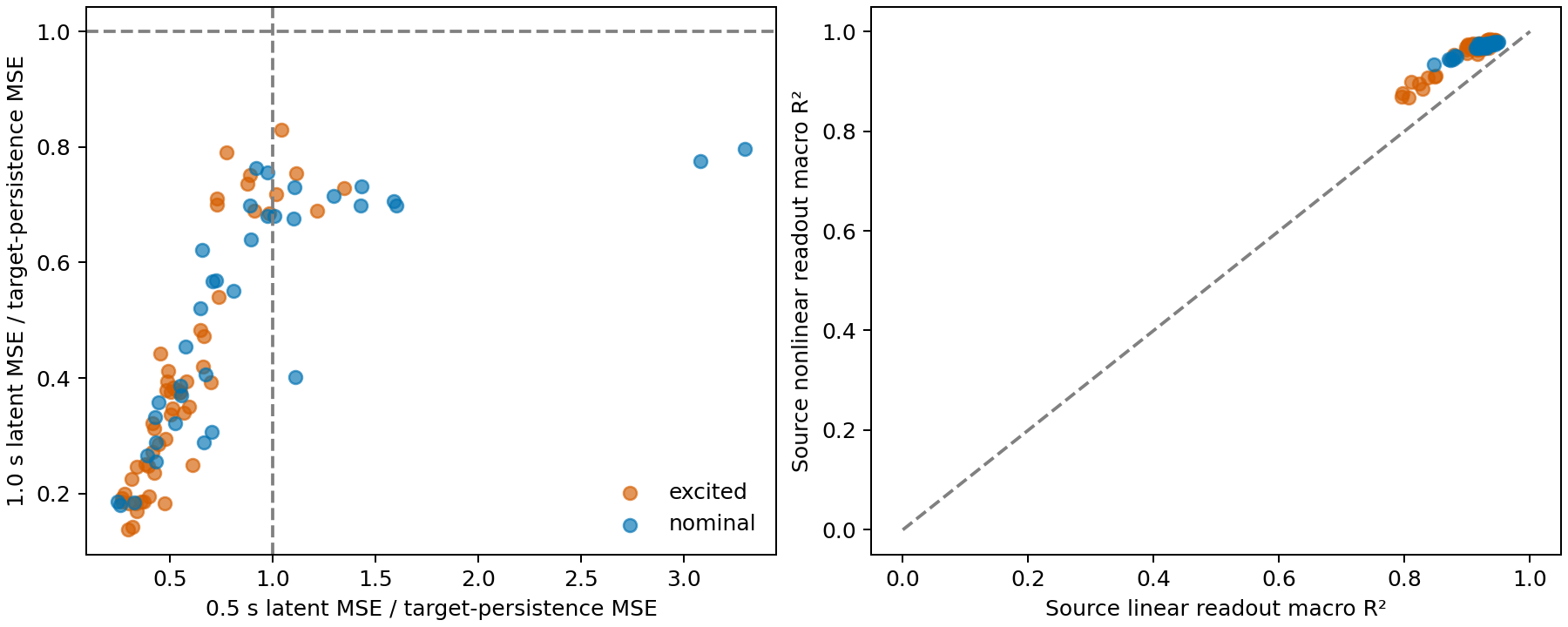}{Predictive competence and current-output retention on the common panel. Left: each fit's $0.5$- and $1$-second latent prediction MSE divided by its own target-encoder persistence MSE. Dashed lines mark ratio one. Right: linear versus nonlinear frozen source-readout macro $R^2$ for each fit. Orange and blue denote varied-start/speed and nominal collections.}

\begin{table}[t]
\caption{Median $R^2$ per output on the common panel ($18$ fits per arm). Top: current outputs from the frozen nonlinear source readout. Bottom: one-second decoded forecasts. Position is $(X,Y)$, and $\psi$ is the yaw angle.}
\label{tab:perout}
\centering
\scriptsize
\begin{tabular}{llcccccc}
\toprule
& \textbf{Arm} & $X$ & $Y$ & $\psi$ & $v_x$ & $v_y$ & $r$ \\
\midrule
\multirow{5}{*}{Current} & A & $0.951$ & $0.954$ & $0.975$ & $0.995$ & $0.968$ & $0.987$ \\
& B & $0.942$ & $0.956$ & $0.974$ & $0.994$ & $0.965$ & $0.986$ \\
& C & $0.956$ & $0.953$ & $0.977$ & $0.995$ & $0.970$ & $0.987$ \\
& D & $0.950$ & $0.962$ & $0.973$ & $0.995$ & $0.966$ & $0.986$ \\
& E & $0.960$ & $0.963$ & $0.984$ & $0.996$ & $0.964$ & $0.986$ \\
\midrule
\multirow{5}{*}{$1$\,s forecast} & A & $0.828$ & $0.831$ & $0.850$ & $0.952$ & $0.699$ & $0.872$ \\
& B & $0.800$ & $0.876$ & $0.923$ & $0.970$ & $0.744$ & $0.896$ \\
& C & $0.842$ & $0.831$ & $0.865$ & $0.941$ & $0.697$ & $0.895$ \\
& D & $0.677$ & $0.833$ & $0.828$ & $0.903$ & $0.689$ & $0.867$ \\
& E & $0.878$ & $0.900$ & $0.946$ & $0.980$ & $0.748$ & $0.908$ \\
\bottomrule
\end{tabular}
\end{table}

\fig[tbp]{width=0.92\linewidth}{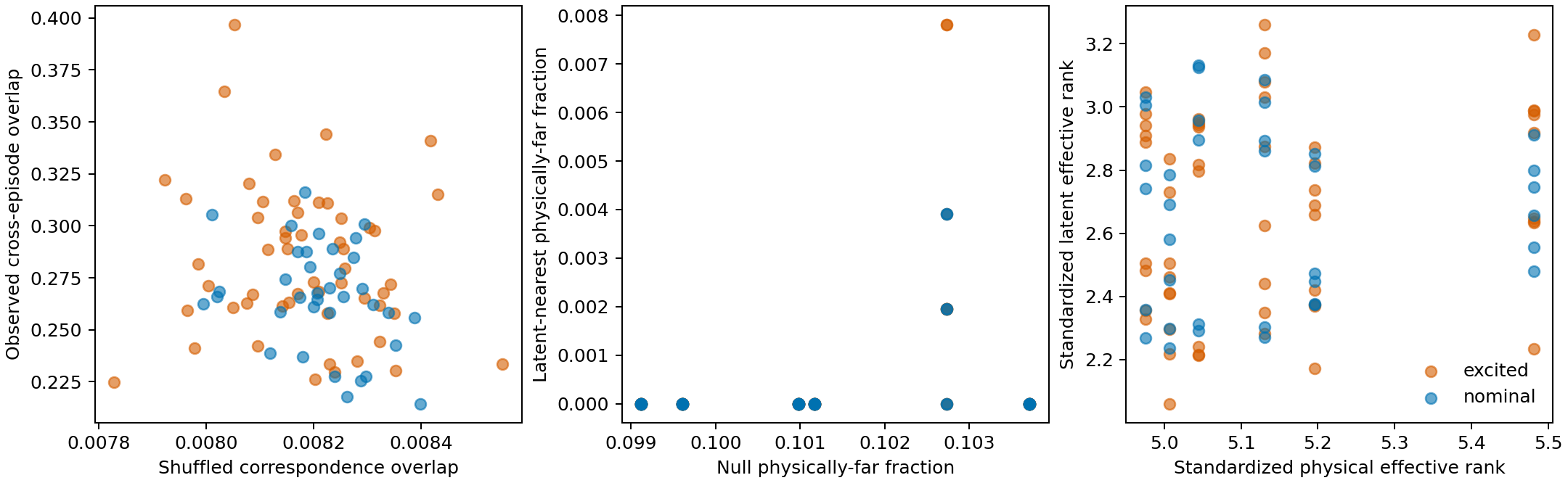}{Geometry diagnostics for individual fits on the common panel. Left: latent/physical neighbor overlap against the shuffled-correspondence overlap. Middle: fraction of physically distant latent-nearest neighbors against the random-pairing fraction. Right: standardized latent versus physical covariance effective rank. Orange and blue denote varied-start/speed and nominal collections.}

\fig[tbp]{width=0.92\linewidth}{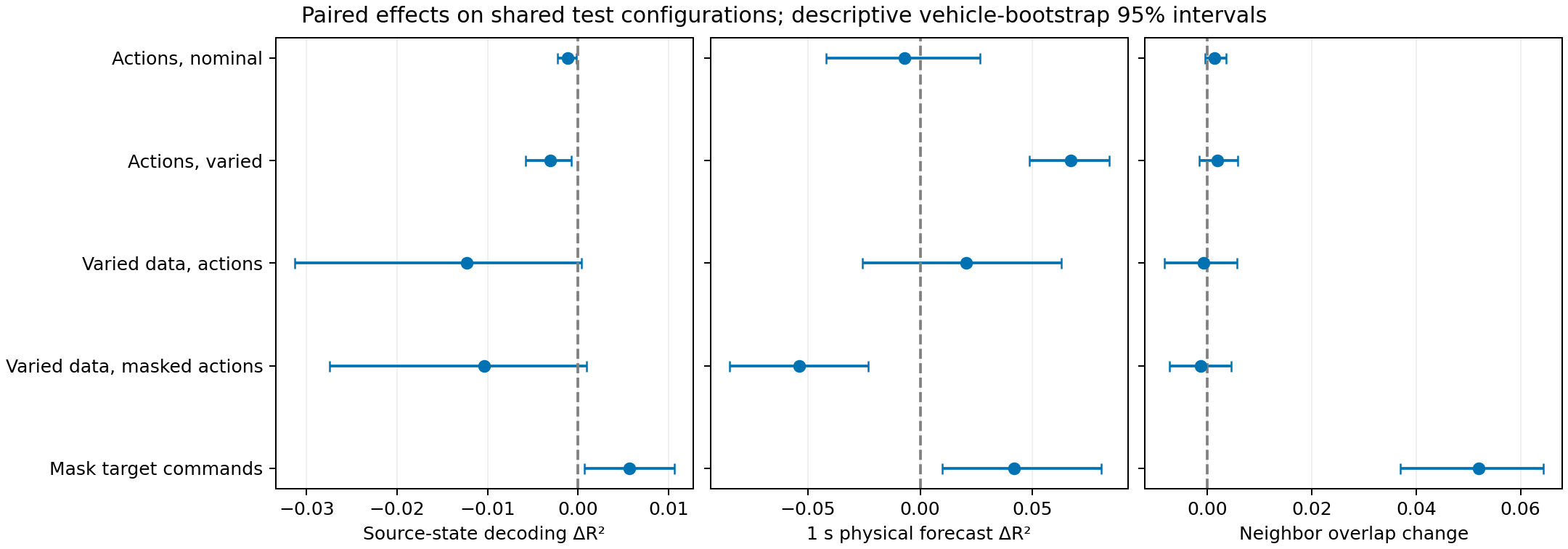}{Paired effects on current-output readout $R^2$, one-second physical forecast $R^2$, and neighbor overlap. Each point is the mean of six configuration-level effects after averaging seeds within configuration, and bars are descriptive 95\% bootstrap intervals over the six configurations. The vertical dashed line marks zero effect.}

\fig[tbp]{width=0.92\linewidth}{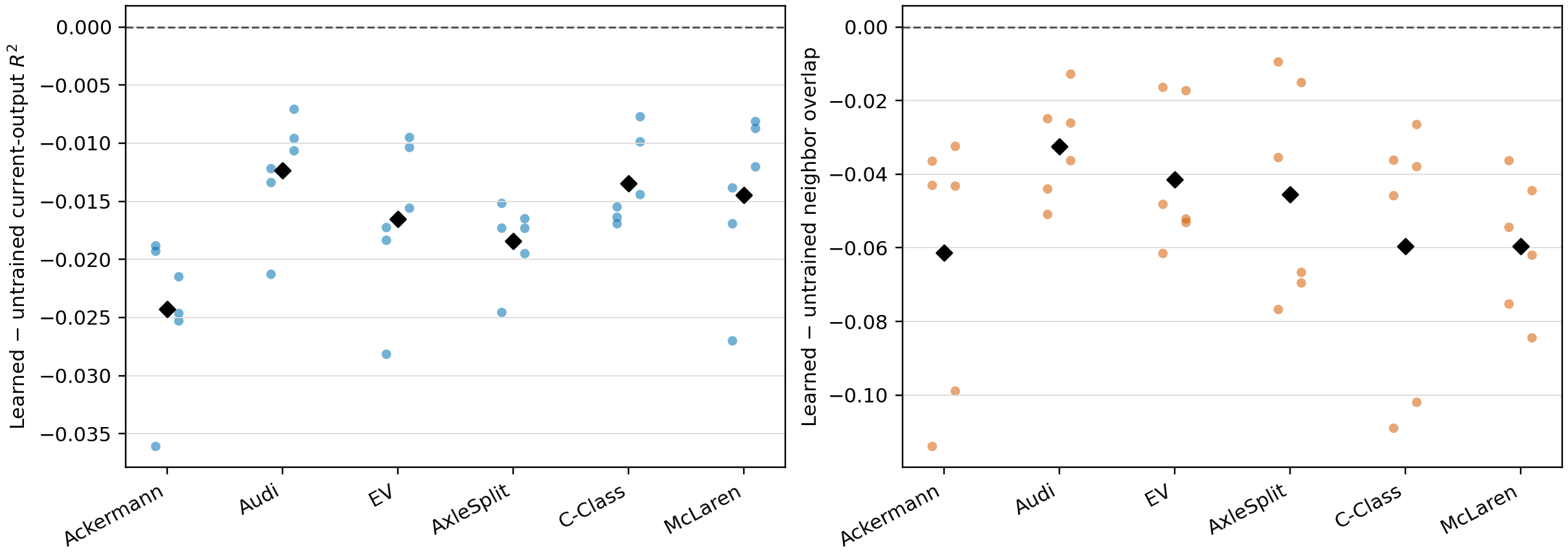}{Trained minus untrained encoder, same architecture, readout, and neighborhood protocol. Left: current-output readout macro $R^2$. Right: physical-neighbor overlap. Points are vehicle/collection/seed pairs for arms A and B on each collection's test episodes, and diamonds are vehicle means. Negative values mean the trained encoder scores below the untrained one.}

\fig[tbp]{width=0.80\linewidth}{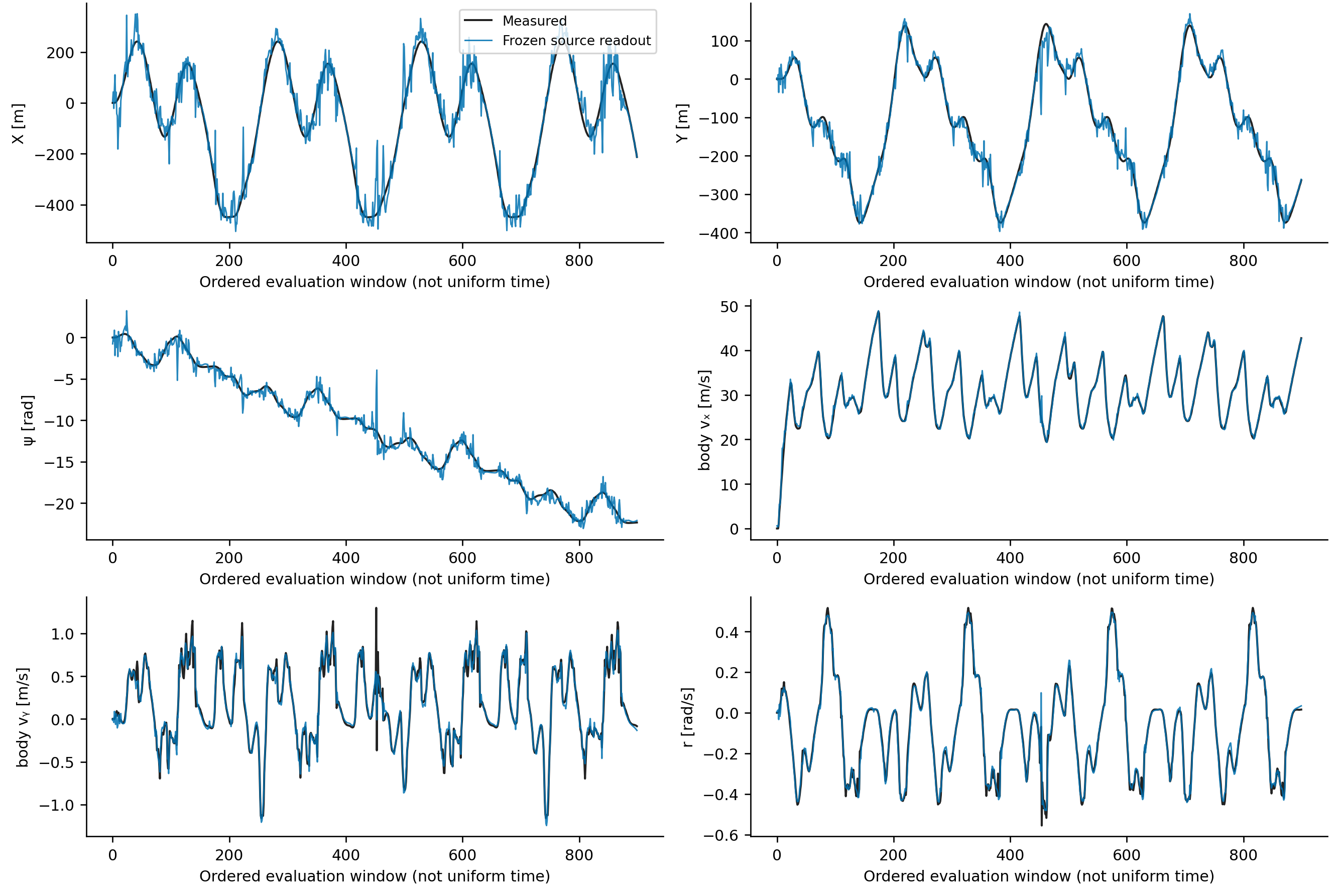}{Current planar outputs measured in CarMaker (black) and recovered by a frozen nonlinear source-latent readout (blue), for the upper median of the $18$ varied-start/speed, action-conditioned fits ranked by own-collection nonlinear decoding macro $R^2$, on the first test episode. The fit was chosen by rank before inspecting traces. Windows are ordered and irregularly spaced, so the horizontal axis is the window index. Training used no reconstruction objective.}

\fig[tbp]{width=\linewidth}{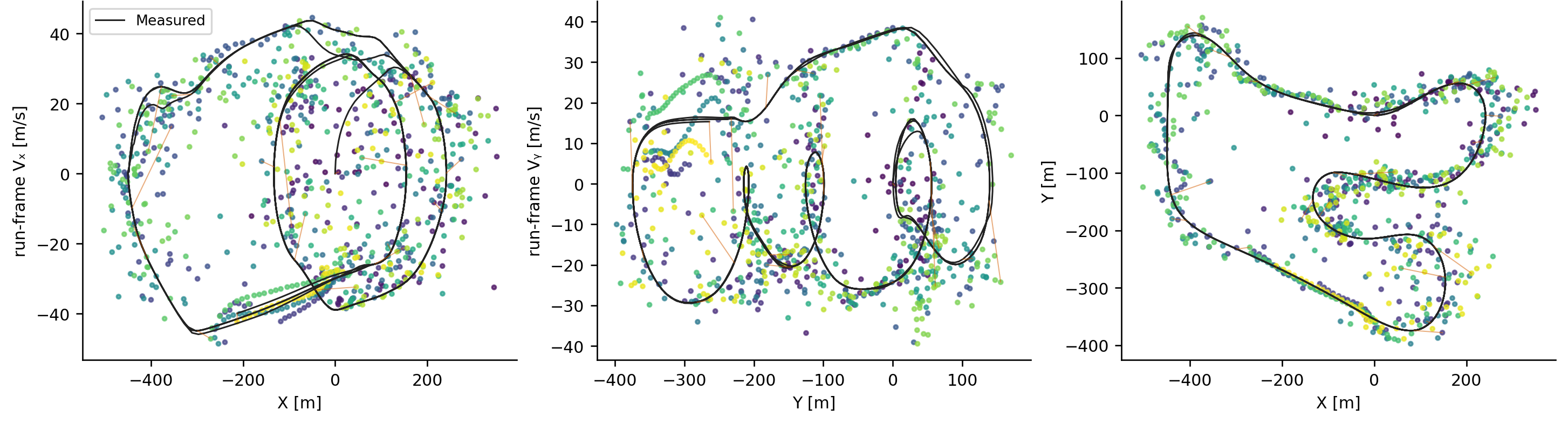}{Measured position--velocity projections (black) and jointly decoded projections (points, colored by episode order) for the fit of Figure~\ref{fig:B_excited_action_target_median_traces.png}. Body velocities are rotated into the run-initial frame with measured and decoded yaw respectively, $V_x=\cos\psi\,v_x-\sin\psi\,v_y$ and $V_y=\sin\psi\,v_x+\cos\psi\,v_y$, and connectors mark selected matched-point discrepancies. Agreement of projected shapes concerns values only, and injectivity and dynamics are tested in Sections~\ref{sec:results} and~\ref{sec:response}.}

\fig[tbp]{width=\linewidth}{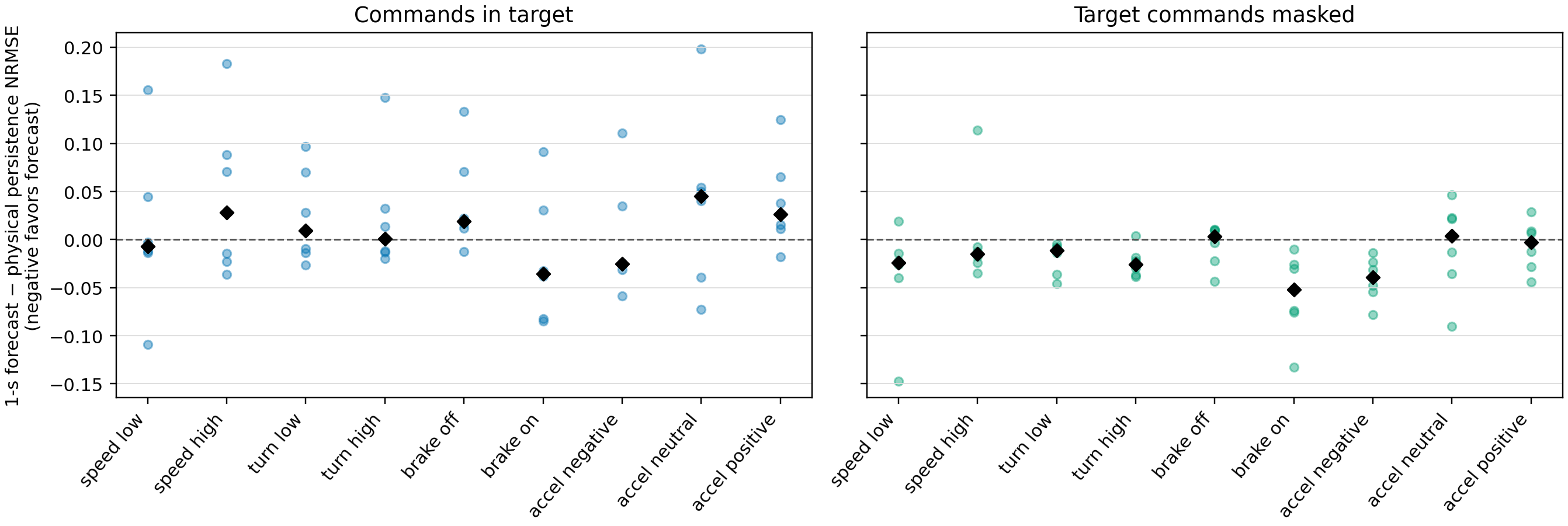}{One-second decoded forecast NRMSE minus physical-persistence NRMSE by operating regime on the common panel. Negative values favor the forecast. Left: arm B (commands in targets). Right: arm E (target commands masked). Points are vehicle-level seed averages, and diamonds are medians. Regimes are fixed before scoring. Speed and absolute yaw rate are split at their training medians, braking is brake-pedal activity above $0.02$, and deceleration, neutral, and acceleration are longitudinal acceleration below $-0.1$, between $-0.1$ and $0.1$, and above $0.1$\,m/s$^2$. The neutral bin holds $54$ to $113$ windows per fit.}

\begin{table}[t]
\caption{Coverage of simulator reference responses. Accepted cases pass the acceptance criteria. Causal nulls are late future pulses at $0.5$\,s. Unresolved cases are pedal pulses below the resolution floor or responses that change by more than $15\%$ across pulse sizes.}
\label{tab:coverage}
\centering
\scriptsize
\begin{tabular}{lccc}
\toprule
\textbf{Family / horizon} & \textbf{Accepted} & \textbf{Causal null} & \textbf{Unresolved} \\
\midrule
future, $0.5$\,s & $14/36$ & $18$ & $4$ \\
future, $1$\,s & $26/36$ & $0$ & $10$ \\
history, $0.5$\,s & $24/36$ & $0$ & $12$ \\
history, $1$\,s & $23/36$ & $0$ & $13$ \\
\bottomrule
\end{tabular}
\end{table}

\fig[tbp]{width=0.92\linewidth}{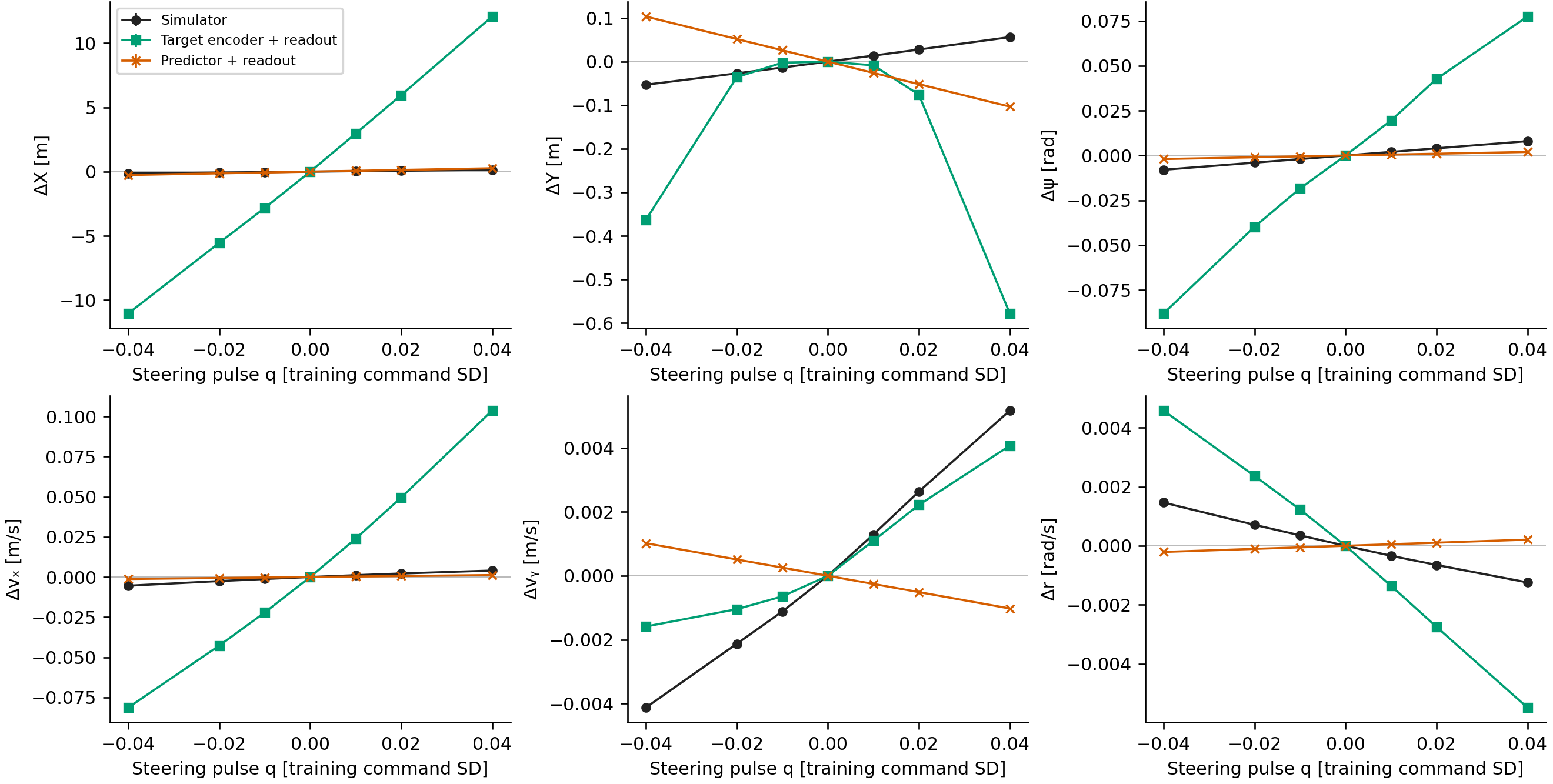}{Future steering finite-pulse responses for a fixed example chosen before inspecting accuracy (Ackermann, arm B, seed $17$, early steering, $1$\,s endpoint). Curves show changes from the matched zero run as pulse amplitude varies, for the simulator (black), the target and readout path (green), and the predictive path (orange), with repeat standard deviations as error bars (repeated runs produced identical saved endpoint values, so the bars have zero height and are not visible). Aggregate results for all fits are in Table~\ref{tab:agree}.}

\begin{table}[t]
\caption{Local value competence at $1$\,s on the simulator-timed families. Medians across six vehicles and three seeds. A latent ratio below one beats matched target-encoder persistence. Physical wins count fits whose forecast error is below their own persistence error, so they compare paired errors.}
\label{tab:localval}
\centering
\scriptsize
\begin{tabular}{llcccccc}
\toprule
\textbf{Arm} & \textbf{Family} & \textbf{Lat.\ ratio} & \textbf{Tgt NRMSE} & \textbf{Fcst NRMSE} & \textbf{Pers.\ NRMSE} & \textbf{Lat.\ wins} & \textbf{Phys.\ wins} \\
\midrule
A & future & $0.324$ & $0.055$ & $0.347$ & $0.289$ & $15/18$ & $6/18$ \\
A & history & $0.361$ & $0.081$ & $0.374$ & $0.270$ & $15/18$ & $5/18$ \\
B & future & $0.162$ & $0.057$ & $0.288$ & $0.289$ & $17/18$ & $11/18$ \\
B & history & $0.228$ & $0.074$ & $0.327$ & $0.270$ & $17/18$ & $8/18$ \\
C & future & $0.393$ & $0.044$ & $0.265$ & $0.289$ & $15/18$ & $9/18$ \\
C & history & $0.395$ & $0.060$ & $0.373$ & $0.270$ & $14/18$ & $7/18$ \\
D & future & $0.539$ & $0.053$ & $0.390$ & $0.289$ & $14/18$ & $7/18$ \\
D & history & $0.491$ & $0.066$ & $0.408$ & $0.270$ & $13/18$ & $7/18$ \\
E & future & $0.064$ & $0.059$ & $0.260$ & $0.289$ & $18/18$ & $13/18$ \\
E & history & $0.095$ & $0.062$ & $0.231$ & $0.270$ & $18/18$ & $12/18$ \\
\bottomrule
\end{tabular}
\end{table}

\begin{table}[t]
\caption{Decoder-free response comparison at $1$\,s: predicted versus realized target-latent response to the same pulses, in training-standardized target coordinates. Medians of vehicle-level medians. A norm ratio of one means equal magnitude and says nothing about direction.}
\label{tab:latent}
\centering
\scriptsize
\begin{tabular}{llccc}
\toprule
\textbf{Arm} & \textbf{Family} & \textbf{Norm ratio (pred./realized)} & \textbf{Rel.\ error} & \textbf{Cosine} \\
\midrule
B & future & $0.055$ & $0.995$ & $0.312$ \\
E & future & $0.174$ & $0.990$ & $0.355$ \\
B & history & $7.664$ & $7.536$ & $0.263$ \\
E & history & $5.487$ & $5.418$ & $0.387$ \\
\bottomrule
\end{tabular}
\end{table}

\fig[tbp]{width=\linewidth}{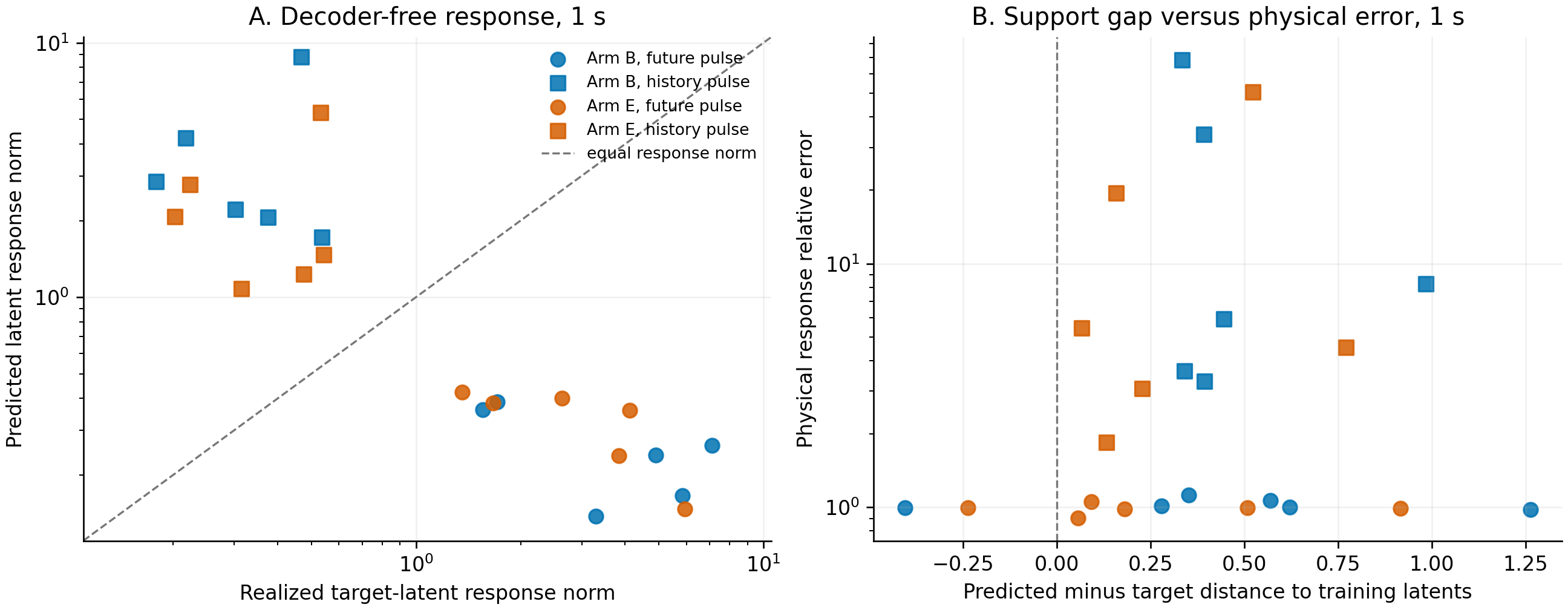}{Decoder-free response at $1$\,s. Left: norm of the predicted latent response against the realized target-latent response to the same pulse (vehicle medians). The dashed line marks equal norm, and direction agreement is in Table~\ref{tab:latent}. Right: predicted-minus-realized distance to sampled training target latents against physical response error. Predicted latents are usually farther from training targets (median ratios $1.16$--$1.44$), but the distance does not order the physical errors (within-fit rank correlations between $-0.09$ and $0.12$).}

\fig[tbp]{width=0.55\linewidth}{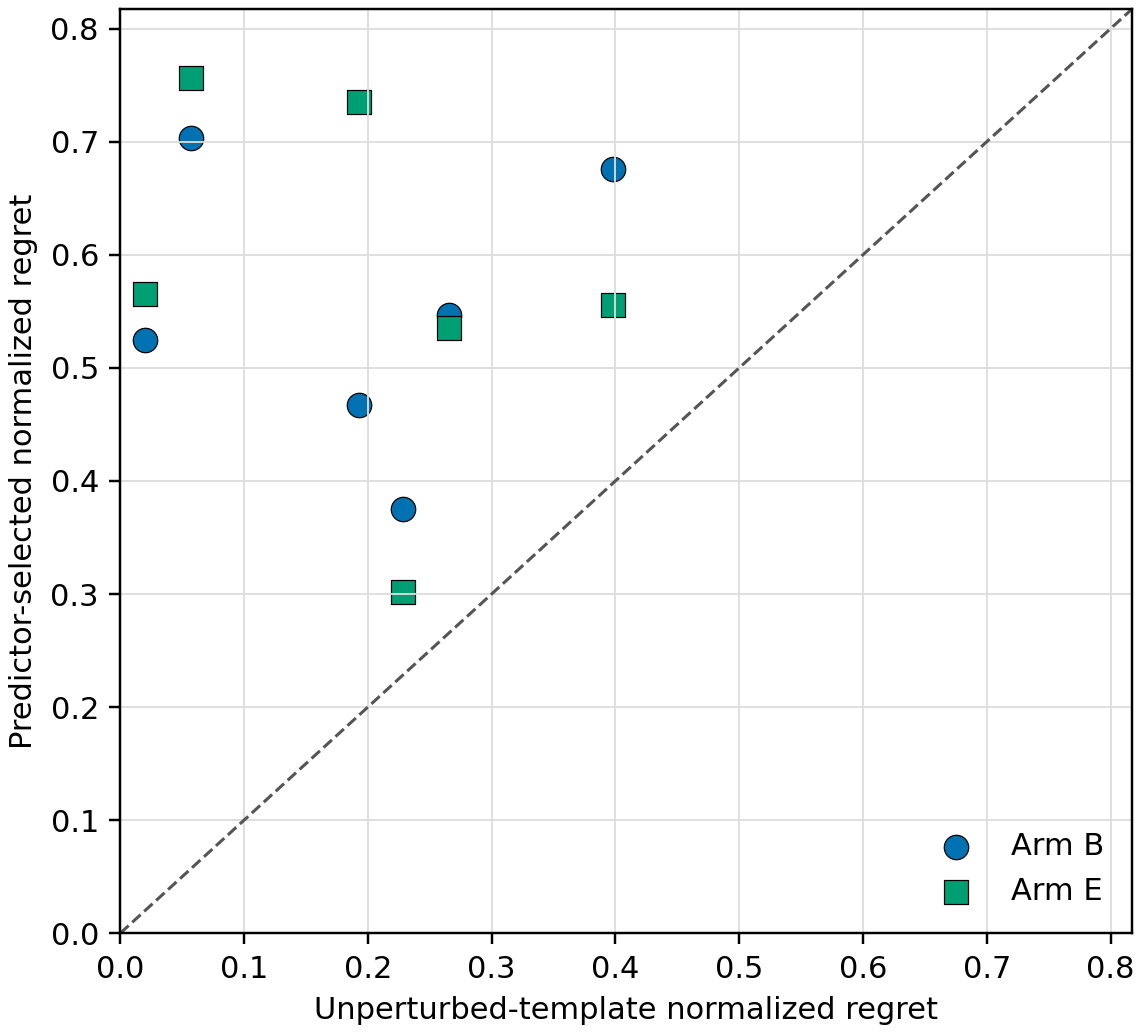}{Finite-candidate ranking. Vehicle-mean normalized regret of the model-selected command (vertical) against keeping the unperturbed template (horizontal), for arms B and E. Each point combines three seeds and the accepted cases of one vehicle. Points above the diagonal mean the model choice was worse than the template.}

\fig[tbp]{width=0.92\linewidth}{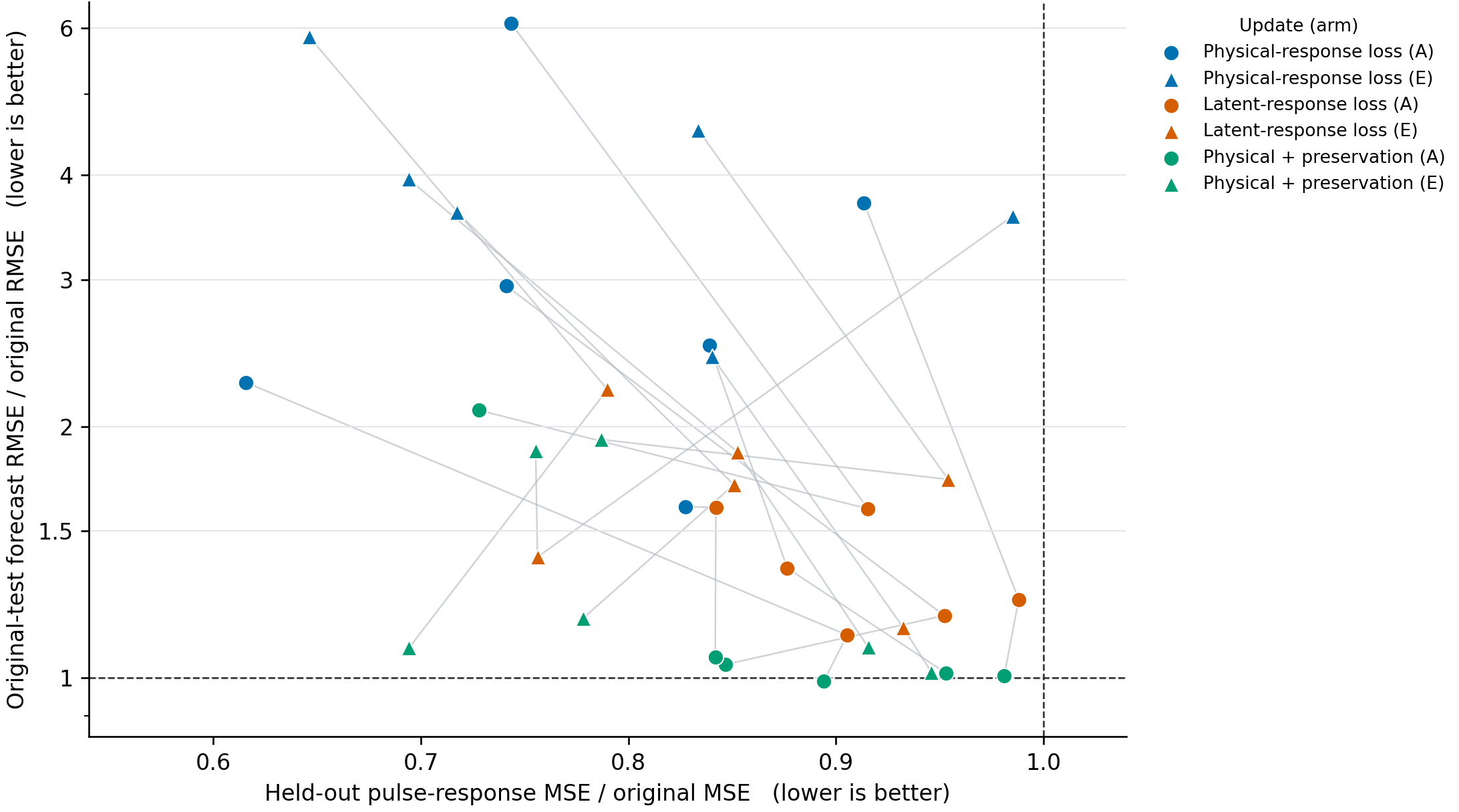}{Response-targeted predictor updates for the seed-$17$ fits of arms A and E. Horizontal: held-out pulse-response MSE after the update divided by before, at the independent anchor. Vertical: one-second forecast RMSE on the original test episodes after divided by before. Colors mark the update objective, and gray segments join the three updates of the same fit. The forecast-preserving variant adds a penalty on changes to the original forecasts.}

\paragraph{Response-targeted updates.}
The physical-response and latent-response updates lower the held-out response error with median ratios of $0.785$ and $0.891$ and raise the one-second forecast RMSE by median factors of $3.59$ and $1.49$. Adding the preservation penalty keeps the held-out response improvement ($12/12$ fits, median ratio $0.844$, $61/104$ resolved cases improved) and reduces the forecast cost. Original-test forecast RMSE still increases in $11/12$ fits, with median ratio $1.073$ and maximum $2.092$. One McLaren fit improves both. Because the forecast-preserving variant was formulated in response to the observed tradeoff, we treat it as exploratory, and a confirmatory evaluation would require a prespecified protocol and a fresh anchor.

\fig[tbp]{width=0.92\linewidth}{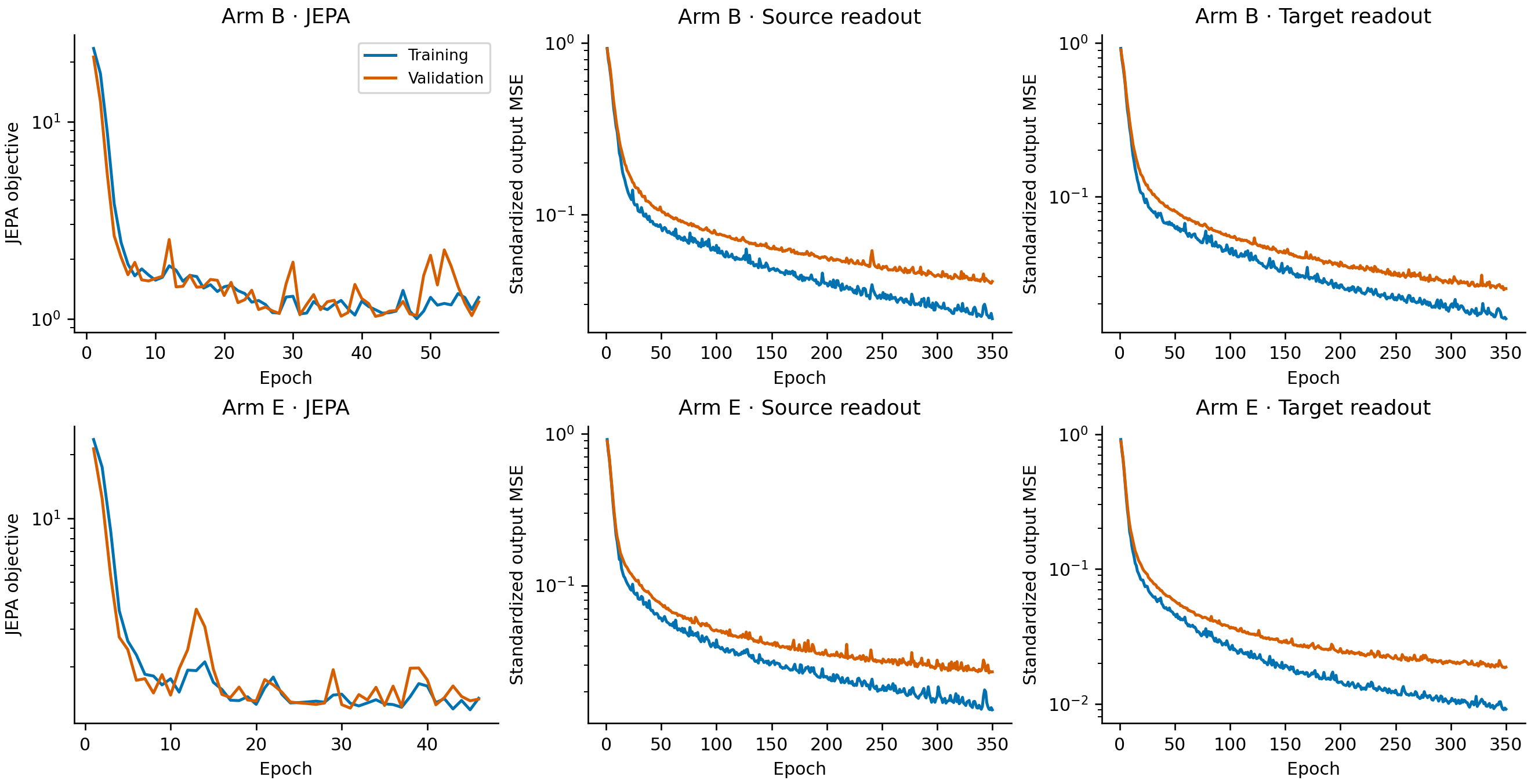}{Train and validation learning curves for the Ackermann, seed-$17$ models in arms B and E, with source and target readouts shown separately. Loss values live in each fit's own target space and are comparable within a fit only.}

\paragraph{Response robustness.}
In arm E, the $1$\,s future/history predictor relative errors are $0.995/5.103$ on the full accepted set, $0.995/5.025$ when both learned paths also pass the scale test, and $0.996/4.985$ among cases beating both local latent and physical persistence. Joint scale agreement holds for both learned paths in about $93.7\%$ of physically accepted comparison rows. The last subset has $53$ future and $40$ history fit/case rows with uneven vehicle and direction coverage. Absolute errors accompany ratios because a small physical response inflates a relative history error.

\paragraph{Sensitivity of the response comparison.}
Table~\ref{tab:gates} varies the acceptance criteria of Section~\ref{sec:metrics}. The median predictor errors of arms B and E change by at most $0.03$ for future pulses and not at all for history pulses. For accepted one-second cases, the predictor relative error at pulse sizes $0.01$, $0.02$, and $0.04$ is $1.022$ at each size for arm B and $0.995$ for arm E on future pulses, and rises from $8.19$ to $8.33$ (B) and from $5.10$ to $5.45$ (E) on history pulses. On the $18$ scale-inconsistent cases, the smallest-amplitude secant gives future-pulse errors of $0.86$ to $0.98$ and history-pulse errors of $13$ to $48$, an exploratory check without scale certification.

\begin{table}[h]
\caption{Accepted cases and vehicle-balanced one-second predictor relative error (future / history) under alternative acceptance criteria. The default is a $15\%$ scale threshold and a $10^{-6}$ norm floor.}
\label{tab:gates}
\centering
\scriptsize
\begin{tabular}{lcccc}
\toprule
\textbf{Criterion} & \textbf{Future $0.5$/$1$\,s} & \textbf{History $0.5$/$1$\,s} & \textbf{Arm B error} & \textbf{Arm E error} \\
\midrule
Scale $10\%$ & $14/23$ & $24/23$ & $0.996$ / $8.193$ & $0.995$ / $5.103$ \\
Scale $15\%$ (default) & $14/26$ & $24/23$ & $1.022$ / $8.193$ & $0.995$ / $5.103$ \\
Scale $20\%$ & $15/28$ & $26/25$ & $1.022$ / $8.193$ & $0.995$ / $5.103$ \\
Scale $25\%$ & $15/29$ & $26/26$ & $1.024$ / $8.193$ & $0.995$ / $5.103$ \\
Scale $30\%$ & $15/29$ & $27/26$ & $1.024$ / $8.193$ & $0.995$ / $5.103$ \\
Norm floor $10^{-7}$ to $10^{-5}$ & $14/26$ & $24/23$ & $1.022$ / $8.193$ & $0.995$ / $5.103$ \\
\bottomrule
\end{tabular}
\end{table}

\paragraph{Second anchor.}
At $2250$\,m and $180$\,km/h the simulator provides future steering, gas, and brake pulses over the first half-second at one amplitude ($0.04$ command standard deviations), with central or one-sided stencils as actuator bounds permit. These are secants without the scale check. Of $36$ direction and horizon cases, $34$ exceed the norm floor. Table~\ref{tab:anchor2} scores all $90$ original fits, averaging seeds within each case. Arms C and D give zero predictor response by construction, with relative error $1$ and an undefined cosine.

\begin{table}[h]
\caption{Second-anchor response agreement for the original fits (future pulses). Predictor and target rows compare decoded responses with the simulator. Latent rows compare predicted with realized target-latent responses, and their gain is the norm ratio.}
\label{tab:anchor2}
\centering
\scriptsize
\begin{tabular}{llcccccc}
\toprule
& & \multicolumn{3}{c}{\textbf{$0.5$\,s}} & \multicolumn{3}{c}{\textbf{$1$\,s}} \\
\textbf{Arm} & \textbf{Path} & Rel.\ err. & Cosine & Gain & Rel.\ err. & Cosine & Gain \\
\midrule
A & predictor & $1.004$ & $0.102$ & $0.155$ & $1.030$ & $0.167$ & $0.203$ \\
A & target & $1.557$ & $0.711$ & $1.675$ & $1.081$ & $0.549$ & $1.458$ \\
B & predictor & $1.007$ & $0.153$ & $0.343$ & $1.046$ & $0.146$ & $0.342$ \\
B & target & $1.847$ & $0.725$ & $2.095$ & $1.058$ & $0.572$ & $1.394$ \\
C & target & $1.651$ & $0.671$ & $1.775$ & $1.247$ & $0.534$ & $1.685$ \\
D & target & $1.687$ & $0.694$ & $1.890$ & $1.255$ & $0.601$ & $1.416$ \\
E & predictor & $0.986$ & $0.260$ & $0.358$ & $0.993$ & $0.210$ & $0.416$ \\
E & target & $0.831$ & $0.696$ & $1.223$ & $1.015$ & $0.519$ & $1.281$ \\
B & latent & $0.953$ & $0.461$ & $0.113$ & $1.013$ & $0.120$ & $0.319$ \\
E & latent & $0.978$ & $0.407$ & $0.252$ & $0.964$ & $0.329$ & $0.380$ \\
\bottomrule
\end{tabular}
\end{table}

\paragraph{Supervised forecaster responses.}
Table~\ref{tab:supervised} scores the supervised forecaster on the pulses and accepted cases of Section~\ref{sec:response}, with the same stencils and aggregation (one seed, so no seed averaging), next to the matched seed-$17$ predictors of arms A and B trained on the same collections. At the zero pulse member its median six-output NRMSE is $0.044$ and $0.065$ at $0.5$ and $1$\,s for the future family and $0.064$ and $0.086$ for the history family.

\begin{table}[h]
\caption{Pulse responses of the supervised forecaster (Sup.) and of the matched latent predictors (Pred.), seed $17$, vehicle-balanced medians.}
\label{tab:supervised}
\centering
\scriptsize
\begin{tabular}{lllcccccc}
\toprule
& & & \multicolumn{3}{c}{\textbf{Sup.}} & \multicolumn{3}{c}{\textbf{Pred.}} \\
\textbf{Coll.} & \textbf{Family} & $h$ & Rel.\ err. & Cosine & Gain & Rel.\ err. & Cosine & Gain \\
\midrule
A & future & $0.5$ & $0.755$ & $0.829$ & $0.774$ & $1.001$ & $0.173$ & $0.194$ \\
A & future & $1$ & $1.058$ & $0.756$ & $1.263$ & $1.001$ & $0.054$ & $0.208$ \\
A & history & $0.5$ & $3.073$ & $0.371$ & $3.246$ & $11.375$ & $0.014$ & $11.549$ \\
A & history & $1$ & $2.596$ & $0.123$ & $2.083$ & $7.680$ & $0.072$ & $7.653$ \\
B & future & $0.5$ & $0.865$ & $0.744$ & $0.959$ & $1.000$ & $0.123$ & $0.318$ \\
B & future & $1$ & $1.017$ & $0.757$ & $1.357$ & $1.005$ & $0.154$ & $0.290$ \\
B & history & $0.5$ & $2.681$ & $0.344$ & $2.835$ & $11.319$ & $0.179$ & $11.716$ \\
B & history & $1$ & $3.538$ & $0.028$ & $3.009$ & $8.091$ & $0.109$ & $8.140$ \\
\bottomrule
\end{tabular}
\end{table}

\paragraph{Readout capacity.}
For the $12$ seed-$17$ pairs of arms A and B on the common panel, Table~\ref{tab:capacity} repeats the source readout at larger capacity, with identical recipes for trained and untrained latents. A ridge readout on the untrained latent reaches median $R^2$ $0.928$, close to the trained linear readouts, and the same MLP on the raw standardized history reaches $0.999$, the ceiling set by the outputs being inputs.

\begin{table}[h]
\caption{Median source-readout macro $R^2$ at three readout capacities, seed-$17$ fits of arms A and B ($12$ pairs).}
\label{tab:capacity}
\centering
\scriptsize
\begin{tabular}{lcccc}
\toprule
\textbf{Readout} & \textbf{Trained} & \textbf{Untrained} & \textbf{Paired gap} & \textbf{Untrained higher} \\
\midrule
GELU $2\times128$, $350$ epochs & $0.974$ & $0.988$ & $0.015$ & $12/12$ \\
GELU $2\times256$, $1000$ epochs & $0.982$ & $0.991$ & $0.008$ & $12/12$ \\
GELU $3\times512$, $1000$ epochs & $0.985$ & $0.992$ & $0.005$ & $12/12$ \\
\bottomrule
\end{tabular}
\end{table}

\paragraph{Geometry sensitivity.}
For the seed-$17$ fits of all arms on the common panel, Table~\ref{tab:geomsens} varies the neighborhood size. The shuffled-correspondence null is about $0.008$ at $k=30$. The collision fraction is zero at the median for trained and untrained latents at the $80$th, $90$th, and $95$th percentile thresholds (largest per-fit value $0.002$), against random-pairing references of about $0.20$, $0.10$, and $0.05$.

\begin{table}[h]
\caption{Common-panel neighbor overlap at $k=30$ (seed-$17$ medians) and the number of vehicles in which the untrained encoder scores higher at $k=10$, $30$, and $100$.}
\label{tab:geomsens}
\centering
\scriptsize
\begin{tabular}{lccc}
\toprule
\textbf{Arm} & \textbf{Trained} & \textbf{Untrained} & \textbf{Untrained higher ($k=10/30/100$)} \\
\midrule
A & $0.241$ & $0.322$ & $6/6$, $6/6$, $6/6$ \\
B & $0.247$ & $0.328$ & $6/6$, $6/6$, $6/6$ \\
C & $0.232$ & $0.322$ & $6/6$, $6/6$, $6/6$ \\
D & $0.243$ & $0.328$ & $6/6$, $6/6$, $6/6$ \\
E & $0.317$ & $0.328$ & $3/6$, $3/6$, $5/6$ \\
\bottomrule
\end{tabular}
\end{table}

\paragraph{Target content.}
Ridge regressions fitted on training windows, selected on validation windows, and scored on the common panel predict the $64$ standardized target coordinates from the future commands, from the source latent, or from both, and predict the future-window commands from the target latent (Table~\ref{tab:probe}, all seeds of arms B and E). The probes measure linear predictability of the targets and leave open which part the predictor uses.

\begin{table}[h]
\caption{Target-content probes, median $R^2$ at $0.5$ / $1$\,s, and the B$-$E vehicle-effect median with its sign count.}
\label{tab:probe}
\centering
\scriptsize
\begin{tabular}{lccc}
\toprule
\textbf{Probe} & \textbf{Arm B} & \textbf{Arm E} & \textbf{B$-$E} \\
\midrule
Target latent from future commands & $0.535$ / $0.560$ & $0.354$ / $0.387$ & $+0.194$ / $+0.187$, $6/6$ positive \\
Target latent from source latent & $0.954$ / $0.889$ & $0.970$ / $0.941$ & $-0.017$ / $-0.045$, $6/6$ negative \\
Target latent from both & $0.973$ / $0.953$ & $0.978$ / $0.963$ & $-0.006$ / $-0.009$, $6/6$ negative \\
Future commands from target latent & $0.929$ / $0.933$ & $0.888$ / $0.889$ & $+0.045$ / $+0.044$, $6/6$ positive \\
\bottomrule
\end{tabular}
\end{table}

\end{document}